\documentclass[twoside,11pt]{article}

\usepackage{jmlr2e}

\usepackage{amsmath,amsfonts}
\usepackage{bm}
\usepackage{dsfont}

\def\eqref#1{equation~\ref{#1}}

\def\1{\bm{1}}

\def\rmI{{\mathbf{I}}}

\DeclareMathAlphabet{\mathsfit}{\encodingdefault}{\sfdefault}{m}{sl}
\SetMathAlphabet{\mathsfit}{bold}{\encodingdefault}{\sfdefault}{bx}{n}

\def\sR{{\mathbb{R}}}

\newcommand{\E}{\mathbb{E}}

\newcommand{\Var}{\mathrm{Var}}

\newcommand{\Cov}{\mathrm{Cov}}

\DeclareMathOperator{\Tr}{Tr}

\def\paref#1{(\ref{#1})}

\newcommand{\diff}{\mathrm{d}}

\newtheorem{assumption}{Assumption}

\usepackage{booktabs}
\usepackage{xcolor}
\usepackage{tikz}
\usetikzlibrary{arrows,positioning}
\usepackage{hyperref}
\usepackage{comment, float}

\numberwithin{equation}{section}

\usepackage{lastpage}

\ShortHeadings{Gaussian Interpolation Flows} 
{Y. Gao, J. Huang, and Y. Jiao}
\firstpageno{1}

\begin{document}

\title{Gaussian Interpolation Flows}

\author{\name Yuan Gao \email yuan0.gao@connect.polyu.hk \\
        \addr Department of Applied Mathematics \\
        The Hong Kong Polytechnic University \\
        Hong Kong SAR, China \\
        \AND
        \name Jian Huang \email j.huang@polyu.edu.hk \\
        \addr Departments of Data Science and AI, and  Applied Mathematics \\
        The Hong Kong Polytechnic University \\
        Hong Kong SAR, China \\
        \AND
        \name Yuling Jiao \email yulingjiaomath@whu.edu.cn \\
        \addr School of Mathematics and Statistics \\
        and Hubei Key Laboratory of Computational Science \\
        Wuhan University, Wuhan, China
}

\editor{} 

\maketitle

\begin{abstract}
Gaussian denoising has emerged as a powerful method for constructing simulation-free continuous normalizing flows for generative modeling. Despite their empirical successes, theoretical properties of these flows and the regularizing effect of Gaussian denoising have remained largely unexplored. In this work, we aim to address this gap by investigating the well-posedness of simulation-free continuous normalizing flows built on Gaussian denoising. Through a unified framework termed Gaussian interpolation flow, we establish the Lipschitz regularity of the flow velocity field, the existence and uniqueness of the flow, and the Lipschitz continuity of the flow map and the time-reversed flow map for several rich classes of target distributions. This analysis also sheds light on the auto-encoding and cycle consistency properties of Gaussian interpolation flows. Additionally, we study the stability of these flows in source distributions and perturbations of the velocity field, using the quadratic Wasserstein distance as a metric. Our findings offer valuable insights into the learning techniques employed in Gaussian interpolation flows for generative modeling, providing a solid theoretical foundation for end-to-end error analyses of learning Gaussian interpolation flows with empirical observations.
\end{abstract}


\section{Introduction}

Generative modeling, which aims to learn the underlying data generating distribution from a finite sample, is a fundamental task in the field of machine learning and statistics \citep{salakhutdinov2015learning}.
Deep generative models (DGMs) find wide-ranging applications across diverse domains such as computer vision, natural language processing, drug discovery, and recommendation systems.
The core objective of DGMs is to learn a nonlinear mapping, either deterministic or stochastic (with outsourcing randomness), which transforms latent samples drawn from a simple reference distribution into samples that closely resemble the target distribution.

Generative adversarial networks (GANs) have emerged as a prominent class of DGMs \citep{goodfellow2014generative, arjovsky2017wasserstein, goodfellow2020generative}.
Through an adversarial training process, GANs learn to approximately generate samples from the data distribution. Variational auto-encoders (VAEs) are another category of DGMs \citep{kingma2014auto, rezende2014stochastic, kingma2019introduction}. In VAEs, the encoding and decoding procedures produce a compressed and structured latent representation, enabling efficient sampling and interpolation. Score-based diffusion models are a promising approach to deep generative modeling that has evolved rapidly since its emergence \citep{song2019generative, song2020improved, ho2020denoising, song2021scorebased, song2021denoising}. The basis of score-based diffusion models lies in the notion of the score function, which characterizes the gradient of the log-density function of a given distribution.

In addition, normalizing flows have gained attention as another powerful class of DGMs \citep{tabak2010density, tabak2013family, kobyzev2020normalizing, papamakarios2021normalizing}.
In normalizing flows, an invertible mapping is learned to transform a simple source distribution into a more complex target distribution by a composition of a series of parameterized, invertible and differentiable intermediate transformations. This framework allows for efficient sampling and training by maximum likelihood estimation \citep{dinh2014nice, rezende2015variational}. Continuous normalizing flows (CNFs) pursue this idea further by performing the transformation over continuous time, enabling fine-grained modeling of dynamic systems from the source distribution to the target distribution. The essence of CNFs lies in defining ordinary differential equations (ODEs) that govern the evolution of CNFs in terms of continuous trajectories. Inspired by the Gaussian denoising approach, which learns a target distribution by denoising its Gaussian smoothed counterpart,
many authors have considered simulation-free estimation methods that have shown great potential in large-scale applications \citep{song2021denoising, liu2023flow, albergo2023building, lipman2023flow, neklyudov2023action, tong2023conditional, chen2023riemannian, albergo2023stochastic, shaul2023kinetic, pooladian2023multisample}. However, despite the empirical success of simulation-free CNFs based on Gaussian denoising, rigorous theoretical analysis of these CNFs have received limited attention thus far.

In this work, we explore an ODE flow-based approach for generative modeling, which we refer to as Gaussian Interpolation Flows (GIFs). This method is derived from the Gaussian stochastic interpolation detailed in Section \ref{sec:gif}. GIFs represent a straightforward extension of the stochastic interpolation method \citep{albergo2023building, liu2023flow, lipman2023flow}. They can be considered a class of CNFs and encompass various ODE flows as special cases.
According to the classical Cauchy-Lipschitz theorem, also known as the Picard-Lindelöf theorem \cite[Theorem 1.1]{hartman2002existence}, a unique solution to the initial value problem for an ODE flow exists if the velocity field is continuous in the time variable and uniformly Lipschitz continuous in the space variable. In the case of GIFs, the velocity field depends on the score function of the push-forward measure. Therefore, it remains to be shown that this velocity field satisfies the regularity conditions stipulated by the Cauchy-Lipschitz theorem.
These regularity conditions are commonly assumed in the literature when analyzing the convergence properties of CNFs or general neural ODEs \citep{chen2018neural, bilovs2021neural, marion2023implicit, marion2023generalization, marzouk2023distribution}. However, there is a theoretical gap in understanding how to translate these regularity conditions on velocity fields into conditions on target distributions.

The main focus of this work is to study and establish the theoretical properties of Gaussian interpolation flow and its corresponding flow map. We show that the regularity conditions of the Cauchy-Lipschitz theorem are satisfied for several rich classes of probability distributions using variance inequalities. Based on the obtained regularity results, we further expose the well-posedness of GIFs, the Lipschitz continuity of flow mappings, and applications to generative modeling. The well-posedness results are crucial for studying the approximation and convergence properties of GIFs learned with the flow or score matching method. When applied to generative modeling, our results further elucidate the auto-encoding and cycle consistency properties exhibited by GIFs.

\bigskip
\begin{figure}[H]
\centering
\begin{tikzpicture}[node distance=1.1cm, auto]
\centering
\tikzset{
    mynode/.style={rectangle, rounded corners, draw=black, top color=white, 
    bottom color=white, 
    very thick, inner sep=0.5em, minimum size=1em, text centered},
    myarrow/.style={-, >=latex', shorten >=1pt, thick},
    mylabel/.style={text width=7em, text centered}
}

\node[] (dummy1) {};

\node[mynode, left=of dummy1] (box-1-1) {\footnotesize{Geometric regularity (Assumption \ref{assump:geom-prop})}};

\node[mynode, right=of dummy1] (box-1-2) {\footnotesize Gaussian interpolation flows};

\node[mynode, below= 2.7cm of box-1-1.west,anchor=west] (box-2-1) {\footnotesize Lipschitz velocity fields (Proposition \ref{prop:vf-bd})};

\node[mynode, below= 2.7cm of box-1-2.east,anchor=east] (box-2-2) {\footnotesize Well-posedness (Theorem \ref{thm:well-posed})};

\node[mynode, below= 2.7cm of box-2-1.west,anchor=west] (box-3-1) {\footnotesize Lipschitz flow maps (Propositions \ref{prop:lip-map}, \ref{prop:lip-map-mog})};

\node[mynode, below= 2.7cm of box-2-2.east,anchor=east] (box-3-2) {\footnotesize Auto-encoding, cycle consistency};

\node[mynode, below= 2.7cm of box-3-1.west,anchor=west] (box-4-1) {\footnotesize Stability in source distributions, stability in velocity fields (Propositions \ref{prop:stab-iv}, \ref{prop:stab-vf})};

\draw[->, >=latex', shorten >=2pt, shorten <=2pt, thick](box-2-1.east) to node[auto, swap, below, text width=6em, text centered] {} (box-2-2.west);

\draw[->, >=latex', shorten >=2pt, shorten <=2pt, thick](box-3-1.east) to node[auto, swap, above, text width=6em, text centered] {} (box-3-2.west);

\draw[->, >=latex', shorten >=2pt, shorten <=2pt, thick](box-1-1.south) to node[auto, swap, left, text width=6em, text centered] {\footnotesize Lemma \ref{lm:cond-cov} Lemma \ref{lm:log-lip} Lemma \ref{lm:vf-op}} (box-1-1.south |- box-2-1.north);

\draw[->, >=latex', shorten >=2pt, shorten <=2pt, thick](box-2-1.south) to node[auto, swap, left, text width=6em, text centered] {\footnotesize Lemma \ref{lm:diff-eq-flow} Lemma \ref{lm:flow-map-Lip-bd}} (box-2-1.south |- box-3-1.north);

\draw[->, >=latex', shorten >=2pt, shorten <=2pt, thick](box-2-2.south) to node[auto, swap, right, text width=6em, text centered] {\footnotesize Corollary \ref{cor:time-reve-flow}} (box-2-2.south |- box-3-2.north);

\draw[->, >=latex', shorten >=2pt, shorten <=2pt, thick](box-1-2.south) to node[auto, swap, right, text width=6em, text centered] {} (box-1-2.south |- box-2-2.north);

\draw[->, >=latex', shorten >=2pt, shorten <=2pt, thick](box-3-1.south) to node[auto, swap, left, text width=6em, text centered] {\footnotesize Lemma \ref{lm:ag-formula} Corollary \ref{cor:prepare-bd} Corollary \ref{cor:change-var}} (box-3-1.south |- box-4-1.north);

\end{tikzpicture}
\caption{Roadmap of the main results.}
\label{fig:roadmap}
\end{figure}
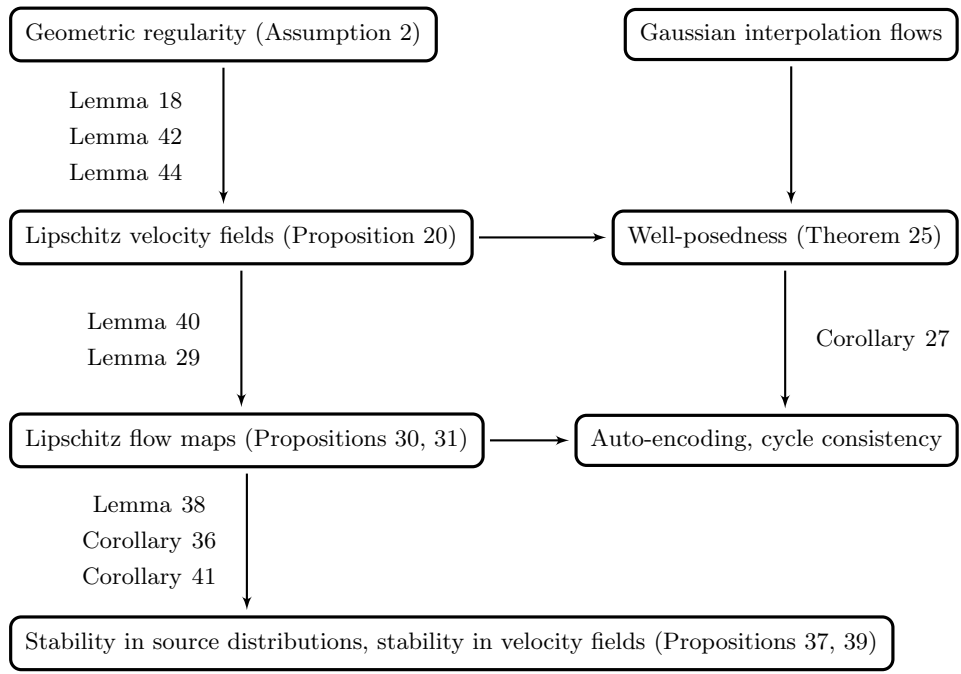

\subsection{Our main contributions}
We provide an overview of the main results in Figure \ref{fig:roadmap}, in which we indicate the assumptions used in our analysis and the relationship between the results. We also summarize our main contributions below.

\begin{itemize}
\item
In Section \ref{sec:gif}, we extend the framework of stochastic interpolation proposed in \citet{albergo2023building}. Various ODE flows can be considered special cases of the extended framework. We prove that the marginal distributions of GIFs satisfy the continuity equation converging to the target distribution in the weak sense. Several explicit formulas of the velocity field and its derivatives are derived, which can facilitate computation and regularity estimation.

\item
In Sections \ref{sec:spatial-lip} and \ref{sec:well-posed}, we establish the spatial Lipschitz regularity of the velocity field for a range of target measures with rich structures, which is sufficient to guarantee the well-posedness of GIFs. Additionally, we deduce the Lipschitz regularity of both the flow map and its time-reversed counterpart. The well-posedness of GIFs is an essential attribute, serving as a foundational requirement for investigating numerical solutions of GIFs. It is important to note that while the flow maps are demonstrated to be Lipschitz continuous transport maps for generative modeling, the Lipschitz regularity for optimal transport maps has only been partially established to date.

\item
In Section \ref{sec:app}, we show that the auto-encoding and cycle consistency properties of GIFs are inherently satisfied when the flow maps exhibit Lipschitz continuity with respect to the spatial variable. This demonstrates that exact auto-encoding and cycle consistency are intrinsic characteristics of GIFs. Our findings lend theoretical support to the findings made by \citet{su2023dual}, as illustrated in Figures \ref{fig:auto-encoding} and \ref{fig:cycle-consistency}.

\item In Section \ref{sec:app}, we conduct the stability analysis of GIFs, examining how they respond
to changes in source distributions and to perturbations in the velocity field. This analysis, conducted in terms of the quadratic Wasserstein distance, provides valuable insights that justify the use of learning techniques such as Gaussian initialization and flow or score matching.
\end{itemize}

\section{Preliminaries}
In this section, we include several preliminary setups to show notations, basic assumptions, and several useful variance inequalities.

\noindent \textbf{Notation.}
Here we summarize the notation.
The space $\sR^d$ is endowed with the Euclidean metric and we denote by $\Vert \cdot \Vert$ and $\langle \cdot, \cdot \rangle$ the corresponding norm and inner product.
Let $\mathbb{S}^{d-1} := \{x \in \sR^d : \Vert x \Vert = 1\}$.
For a matrix $A \in \sR^{k \times d}$, we use $A^{\top}$ for the transpose, and the spectral norm is denoted by $\Vert A \Vert_{2, 2} := \sup_{x \in \mathbb{S}^{d-1}} \Vert Ax \Vert$.
For a square matrix $A \in \sR^{d \times d}$, we use $\det(A)$ for the determinant and $\Tr(A)$ for the trace.
We use $\rmI_d$ to denote the $d \times d$ identity matrix.
For two symmetric matrices $A, B \in \sR^{d \times d}$, we denote $A \succeq B$ or $B \preceq A$ if $A - B$ is positive semi-definite.
For two vectors $x, y \in \sR^d$, we denote $x \otimes y := x y^{\top}$.
For $\Omega_1 \subset \sR^k, \Omega_2 \subset \sR^d, n \ge 1$, we denote by $C^n(\Omega_1; \Omega_2)$ the space of continuous functions $f: \Omega_1 \to \Omega_2$ that are $n$ times differentiable and whose partial derivatives of order $n$ are continuous.
If $\Omega_2 \subset \sR$, we simply write $C^n(\Omega_1)$.
For any $f(x) \in C^2(\sR^d)$, let $\nabla_x f, \nabla^2_x f, \nabla_x \cdot f$, and $\Delta_x f$ denote its gradient, Hessian, divergence, and Laplacian, respectively.
We use $X \lesssim Y$
to denote $X \le C Y$ for some constant $C > 0$.
The function composition operation is marked as $g \circ f := g(f(x))$ for functions $f$ and $g$.

The Borel $\sigma$-algebra of $\sR^d$ is denoted by $\mathcal{B}(\sR^d)$.
The space of probability measures defined on $(\sR^d, \mathcal B(\sR^d))$ is denoted as $\mathcal{P}(\sR^d)$.
For any $\sR^d$-valued random variable $\mathsf{X}$, we use $\E[\mathsf{X}]$ and $\Cov(\mathsf{X})$ to denote its expectation and covariance matrix, respectively.
We use $\mu * \nu$ to denote the convolution for any two probability measures $\mu$ and $\nu$, and we use $\overset{d}{=}$ to indicate two random variables have the same probability distribution.
For a random variable $\mathsf{X}$, let $\mathrm{Law}(\mathsf{X})$ denote its probability distribution.
Let $g: \sR^k \to \sR^d$ be a measurable mapping and $\mu$ be a probability measure on $\sR^k$. The push-forward measure $f_{\#} \mu$ of a measurable set $A$ is defined as $f_{\#} \mu := \mu(f^{-1} (A))$.
Let $N (m, \Sigma)$ denote a $d$-dimensional Gaussian random variable with mean vector $m \in \sR^d$ and covariance matrix $\Sigma \in \sR^{d \times d}$.
For simplicity, let $\gamma_{d, \sigma^2} := N(0, \sigma^2 \rmI_d)$,
and let $\varphi_{m, \sigma^2}(x)$ denote the probability density function of $N (m, \sigma^2 \rmI_d)$ with respect to the Lebesgue measure.
If $m = 0, \sigma = 1$, we abbreviate these as $\gamma_d$ and $\varphi(x)$.
Let $L^p(\sR^d; \sR^{\ell}, \mu)$ denote the $L^p$ space with the $L^p$ norm for $p \in [1, \infty]$ w.r.t. a measure $\mu$.
To simplify the notation, we write $L^p(\sR^d, \mu)$ if $\ell = 1$, $L^p(\sR^d; \sR^{\ell})$ if the Lebesgue measure is used, and $L^p(\sR^d)$ if both hold.

\subsection{Assumptions}

We focus on the probability distributions satisfying several types of assumptions of weak convexity, which offer a geometric notion of regularity that is dimension-free in the study of high-dimensional distributions \citep{klartag2010high}. On one hand, weak-convexity regularity conditions are useful in deriving dimension-free guarantees for generative modeling and sampling from high-dimensional distributions. On the other hand, they accommodate distributions with complex shapes, including those with multiple modes.

\begin{definition} [\citealp{cattiaux2014semi}]
\label{def:semi-log-concave}
A probability measure $\mu (\diff x) = \exp(-U) \diff x$ is $\kappa$-semi-log-concave for some $\kappa \in \sR$ if its support $\Omega \subseteq \sR^d$ is convex and its potential function $U \in C^2(\Omega)$ satisfies
\begin{equation*}
\nabla^2_x U(x) \succeq \kappa \mathbf{I}_d, \quad \forall x \in \Omega.
\end{equation*}
\end{definition}

The $\kappa$-semi-log-concavity condition is a relaxed notion of log-concavity, since here $\kappa < 0$ is allowed. When $\kappa \ge 0$, we are considering a log-concave probability measure that is proved to be unimodal \citep{saumard2014log}. However, when $\kappa < 0$, a $\kappa$-semi-log-concave probability measure can be multimodal.

\begin{definition} [\citealp{eldan2018regularization}]
\label{def:semi-log-convex}
A probability measure $\mu (\diff x) = \exp(-U) \diff x$ is $\beta$-semi-log-convex for some $\beta > 0$ if its support $\Omega \subseteq \sR^d$ is convex and its potential function $U \in C^2(\Omega)$ satisfies
\begin{equation*}
\nabla^2_x U(x) \preceq \beta \mathbf{I}_d, \quad \forall x \in \Omega.
\end{equation*}
\end{definition}

The following definition of $L$-log-Lipschitz continuity is a variant of $L$-Lipschitz continuity. It characterizes a first-order condition on the target function rather than a second-order condition such as $\kappa$-semi-log-concavity and $\beta$-semi-log-convexity in Definitions \ref{def:semi-log-concave} and \ref{def:semi-log-convex}.

\begin{definition}
\label{def:log-lip}
A function $f: \sR^d \to \sR_+$ is $L$-log-Lipschitz continuous if its logarithm is $L$-Lipschitz continuous for some $L \ge 0$.
\end{definition}

Based on the definitions, we present two assumptions on the target distribution.
Assumption \ref{assump:well-defined} concerns the absolute continuity and the moment condition.
Assumption \ref{assump:geom-prop} imposes geometric regularity conditions.

\begin{assumption} \label{assump:well-defined}
    The probability measure $\nu$ is absolutely continuous with respect to the Lebesgue measure and has a finite second moment.
\end{assumption}

\begin{assumption}
 \label{assump:geom-prop}
    Let $D := (1 / \sqrt{2}) \mathrm{diam} (\mathrm{supp}(\nu))$. The probability measure $\nu$ satisfies one or more of the following conditions:
    \begin{itemize}
        \item[(i)]    $\nu$ is $\beta$-semi-log-convex for some $\beta >0$ and $\kappa$-semi-log-concave for some $\kappa > 0$ with $\mathrm{supp}(\nu) = \sR^d$;
        \item[(ii)]   $\nu$ is $\kappa$-semi-log-concave for some $\kappa \in \sR$ with $D \in (0, \infty)$;
        \item[(iii)]  $\nu = \gamma_{d, \sigma^2} * \rho$  where $\rho$ is a probability measure supported on a Euclidean ball of radius $R$ on $\mathbb R^d$;
        \item[(iv)]   $\nu$ is $\beta$-semi-log-convex for some $\beta >0$, $\kappa$-semi-log-concave for some $\kappa \le 0$, and $\frac{\diff \nu}{\diff \gamma_d}(x)$ is $L$-log-Lipschitz in $x$ for some $L \ge 0$ with $\mathrm{supp}(\nu) = \sR^d$.
    \end{itemize}
\end{assumption}

    \textbf{Multimodal distributions.}
    Assumption \ref{assump:geom-prop} enumerates scenarios where probability distributions are endowed with geometric regularity.
    We examine the scenarios and clarify whether they cover multimodal distributions.
    Scenario (i) is referred to as the classical strong log-concavity case ($\kappa > 0$), and thus, describes unimodal distributions.
    Scenario (ii) allows $\kappa \le 0$ and requires that the support is bounded.
    Mixtures of Gaussian distributions are considered in Scenario (iii), and typically are multimodal distributions.
    Scenario (iv) also allows $\kappa \le 0$ when considering a log-Lipschitz perturbation of the standard Gaussian distribution.
    Both Scenario (ii) and Scenario (iv) incorporate multimodal distributions due to the potential
     negative lower bound $\kappa$.

\textbf{Lipschitz score.}
Lipschitz continuity of the score function is a basic regularity assumption on target distributions in the study of sampling algorithms based on Langevin and Hamiltonian dynamics. Even for high-dimensional distributions, this assumption endows a great source of regularity.
For an $L$-Lipschitz score function, its corresponding distribution is both $L$-semi-log-convex and $(-L)$-semi-log-concave for some $L \ge 0.$

\subsection{Variance inequalities}
Variance inequalities like the Brascamp-Lieb inequality and the Cram{\'e}r-Rao inequality are fundamental inequalities for explaining the regularizing effect of Gaussian denoising. Combined with $\kappa$-semi-log-concavity and $\beta$-semi-log-convexity, these inequalities are crucial for deducing the Lipschitz regularity of the velocity fields of GIFs in
Proposition \ref{prop:vf-bd}-(b) and (c).

\begin{lemma} [Brascamp-Lieb inequality]
    \label{lm:bli}
    Let $\mu(\diff x) = \exp(-U(x)) \diff x$ be a probability measure on a convex set $\Omega \subseteq \sR^d$ whose potential function $U: \Omega \to \sR$ is of class $C^2$ and strictly convex.
    Then for every locally Lipschitz function $f \in L^2(\Omega, \mu)$,
    \begin{equation}
       \label{eq:bli-gene}
       \Var_{\mu} (f) \le \E_{\mu} \left[ \langle \nabla_x f, (\nabla^2_x U)^{-1} \nabla_x f \rangle \right].
    \end{equation}
\end{lemma}

When applied to functions of the form $f: x \mapsto \langle x, e \rangle$ for any $e \in \mathbb{S}^{d-1}$, the Brascamp-Lieb inequality yields an upper bound of the covariance matrix
 \begin{equation}
        \label{eq:bli-cov}
        \Cov_{\mu} (\mathsf{X}) \preceq \E_{\mu} \left[ (\nabla^2_x U(x))^{-1} \right]
 \end{equation}
 with equality if $\mathsf{X} \sim N(m, \Sigma)$ with $\Sigma$ positive deﬁnite.

 Under the strong log-concavity condition, that is,  $\mu$ is $\kappa$-semi-log-concave with $\kappa > 0$ and the Euclidean Bakry-{\'E}mery criterion is satisfied \citep{bakry1985diffusions}, the Brascamp-Lieb inequality instantly recovers the Poincar{\'e} inequality (see Definition \ref{def:pi}).

 The Brascamp-Lieb inequality originally appears in \cite[Theorem 4.1]{brascamp1976extensions}.
 Alternative proofs are provided in \citet{bobkov2000brunn, bakry2014analysis, cordero2017transport}. The dimension-free inequality \paref{eq:bli-gene} can be further strengthened to obtain several variants with dimensional improvement.

\begin{lemma} [Cram{\'e}r-Rao inequality]
    \label{lm:cri}
    Let $\mu(\diff x) = \exp(-U(x)) \diff x$ be a probability measure on $\sR^d$ whose potential function $U: \sR^d \to \sR$ is of class $C^2$.
    Then for every $f \in C^1(\sR^d)$,
    \begin{equation}
       \label{eq:cri-gene}
       \Var_{\mu} (f) \ge \langle \E_{\mu }[\nabla_x f], \left( \E_{\mu}[\nabla^2_x U] \right)^{-1} \E_{\mu}[\nabla_x f] \rangle.
    \end{equation}
\end{lemma}

When applied to functions of the form $f: x \mapsto \langle x, e \rangle$ for any $e \in \mathbb{S}^{d-1}$, the Cram{\'e}r-Rao inequality yields a lower bound of the covariance matrix
    \begin{equation}
        \label{eq:cri-cov}
        \Cov_{\mu} (\mathsf{X}) \succeq \left( \E_{\mu} [\nabla^2_x U(x)] \right)^{-1}
    \end{equation}
with equality as well if $\mathsf{X} \sim N(m, \Sigma)$ with $\Sigma$ positive definite.

The Cram{\'e}r-Rao inequality plays a central role in asymptotic statistics as well as in information theory. The inequality \paref{eq:cri-cov} has an alternative derivation from the Cram{\'e}r-Rao bound for the location parameter. For detailed proofs of the Cram{\'e}r-Rao inequality, readers are referred to \citet{chewi2022entropic, dai2023lipschitz}, and the references therein.

\section{Gaussian interpolation flows} \label{sec:gif}

Simulation-free CNFs represent a potent class of generative models based on ODE flows. \citet{albergo2023building} and \citet{albergo2023stochastic} introduce an innovative CNF that is constructed using stochastic interpolation techniques, such as Gaussian denoising. They conduct a thorough investigation of this flow, particularly examining its applications and effectiveness in  generative modeling.

We study the ODE flow and its associated flow map as defined by the Gaussian denoising process. This process has been explored from various perspectives, including diffusion models and stochastic interpolants. Building upon the work of \citet{albergo2023building} and \citet{albergo2023stochastic}, we expand the stochastic interpolant framework by relaxing certain conditions on the functions $a_t$ and $b_t$, offering a more comprehensive perspective on the Gaussian denoising process.

In our generalization, we introduce an adaptive starting point to the stochastic interpolation framework, which allows for greater flexibility in the modeling process. By examining this modified framework, we aim to demonstrate that the Gaussian denoising principle is effectively implemented within the context of stochastic interpolation.

\begin{definition} [Vector interpolation] \label{def:vec-interp}
Let $z \in \sR^d,$  $x_1 \in \sR^d$ be two vectors in the Euclidean space and let $x_0 := a_0 z + b_0 x_1$ with $a_0 > 0, b_0 \ge 0$.
Then we construct an interpolant between $x_0$ and $x_1$ over time $t \in [0, 1]$ through $I_t(x_0, x_1),$ defined by
\begin{equation}
    \label{eq:vec-interp}
    I_t (x_0, x_1) = a_t z + b_t x_1,
\end{equation}
where $a_t, b_t$ satisfy
\begin{align}
\label{eq:abt}
\begin{aligned}
& \dot{a}_t \le 0, \quad \dot{b}_t \ge 0, \quad a_0 > 0, \quad b_0 \ge 0, \quad a_1 = 0, \quad b_1 = 1,\\
& a_t > 0 \ \ \text{for any $t \in (0,1)$}, \quad b_t >0 \ \ \text{for any $t \in (0,1)$}, \\
& a_t, b_t \in C^2([0,1)), \quad a_t^2 \in C^1([0, 1]), \quad b_t \in C^1([0, 1]).
\end{aligned}
\end{align}
\end{definition}

\begin{remark}
Compared with the vector interpolant defined by \citet{albergo2023building} (a.k.a. one-sided interpolant in \citet{albergo2023stochastic}), we extend its definition by relaxing the requirements that $a_0 = 1, b_0 = 0$ with $a_0 > 0, b_0 \ge 0$. This consideration is largely motivated by analyzing the probability flow ODEs of the variance-exploding (VE) SDE and the variance-preserving (VP) SDE \citep{song2021scorebased}.
We illustrate examples of interpolants incorporated by Definition \ref{def:vec-interp} in Table \ref{tab:vec-interp}.
\end{remark}

\begin{remark}
We have eased the smoothness conditions for the functions $a_t$ and $b_t$ required in
\citet{albergo2023building}.
 Specifically, we consider the case where $a_t, b_t \in C^2([0,1))$, $a_t^2 \in C^1([0, 1])$, and $b_t \in C^1([0, 1])$. This relaxation enables us to include the Föllmer flow into our framework, characterized by $a_t = \sqrt{1-t^2}$ and $b_t = t$. It is evident that $a_t = \sqrt{1-t^2}$ does not fulfill the condition $a_t \in C^2([0,1])$, but it does meet the requirements $a_t \in C^2([0,1))$ and $a_t^2 \in C^1([0, 1])$.
\end{remark}

\begin{remark}
    The $C^2$ regularity of $a_t, b_t$ is necessary to derive the regularity of the velocity field $v(t, x)$ in Eq. \paref{eq:vf-expect} concerning the time variable $t$. In addition, the $C^1$ regularity of $a_t^2, b_t$ is sufficient to ensure the Lipschitz regularity of the velocity field $v(t, x)$ in Eq. \paref{eq:vf-expect} concerning the space variable $x$.
\end{remark}

{\color{black}
    A natural generalization of the vector interpolant \paref{eq:vec-interp} is to construct a set interpolant between two convex sets through Minkowski sum, which is common in convex geometry. A set interpolant stimulates the construction of a measure interpolant between a structured source measure and a target measure.

As noted, we can construct a measure interpolation using a Gaussian convolution path.
The measure interpolation is particularly relevant to Gaussian denoising and Gaussian channels in information theory as elucidated in Remark \ref{rm:gauss-channel}. 
Because of this connection with Gaussian denoising,
we call the measure interpolation a Gaussian stochastic interpolation.
The Gaussian stochastic interpolation can be understood as a collection of linear combinations of a standard Gaussian random variable and the target random variable.
The coefficients of the linear combinations vary with time $t \in [0, 1]$ as shown in Definition \ref{def:vec-interp}.
Later in this section, we will show this Gaussian stochastic interpolation can be transformed into a deterministic ODE flow.

Gaussian stochastic interpolation has been investigated from several perspectives in the literature.
The rectified flow has been proposed in \citet{liu2023flow}, and its theoretical connection with optimal transport has been investigated in \citet{liu2022rectified}.
The formulation of the rectified flow is to learn the ODE flow defined by stochastic interpolation with linear time coefficients.
In Section 2.3 of \citet{liu2023flow}, there is a nonlinear extension of the rectified flow in which the linear coefficients are replaced by general nonlinear coefficients.
\citet{albergo2023stochastic} extends the stochastic interpolant framework proposed in \citep{albergo2023building} by considering a linear combination among three random variables.
In Section 3 of \citet{albergo2023stochastic}, the original stochastic interpolant framework is recovered as a one-sided interpolant between the Gaussian distribution and the target distribution.
Moreover, \citet{lipman2023flow} propose a flow matching method which directly learns a Gaussian conditional probability path with a neural ODE.
In Section 4.1 of \citep{lipman2023flow}, the velocity fields of the variance exploding and  variance preserving probability flows are shown as special instances of the flow matching framework.
We summarize these formulations as Gaussian stochastic interpolation by slightly extending the original stochastic interpolant framework.
}

\begin{table}[thb]
    \centering
    \begin{tabular}{cccccc}
    \toprule
    Type & VE & VP & Linear & F{\"o}llmer & Trigonometric \\
    \midrule
    $a_t$ & $\alpha_t$ & $\alpha_t$ & $1-t$ & $\sqrt{1 - t^2}$ & $\cos(\tfrac{\pi}{2}t)$ \\
    $b_t$ & $1$   & $\sqrt{1 - \alpha_t^2}$ & $t$ & $t$   & $\sin(\tfrac{\pi}{2}t)$ \\
    $a_0$ & $\alpha_0$ & $\alpha_0$              & $1$ & $1$   & $1$ \\
    $b_0$ & $1$   & $\sqrt{1 - \alpha_0^2}$ & $0$ & $0$   & $0$ \\
    Source & Convolution & Convolution & $\gamma_d$ & $\gamma_d$ & $\gamma_d$ \\
    \bottomrule
    \end{tabular}
    \caption{Summary of various measure interpolants including VE interpolant \citep{song2021scorebased}, VP interpolant \citep{song2021scorebased}, linear interpolant \citep{liu2023flow}, F{\"o}llmer interpolant \citep{dai2023lipschitz}, and trigonometric interpolant \citep{albergo2023building}. There are two types of source measures including a standard Gaussian distribution $\gamma_d$ and a convoluted distribution consisting of the target distribution and $\gamma_d$.}
    \label{tab:vec-interp}
\end{table}

\begin{definition} [Measure interpolation] \label{def:measure-interp}
Let $\mu = \mathrm{Law}(\mathsf{X}_0)$ and $\nu = \mathrm{Law}(\mathsf{X}_1)$ be two probability measures satisfying $\mathsf{X}_0 = a_0 \mathsf{Z} + b_0 \mathsf{X}_1$ where $\mathsf{Z} \sim \gamma_d:=N(0, \rmI_d)$ is independent from $\mathsf{X}_1$.
We call $(\mathsf{X}_t)_{t \in [0, 1]}$ a Gaussian stochastic interpolation from the source measure $\mu$ to the target measure $\nu$, which is defined through $I_t$ over time interval $[0, 1]$ as follows
\begin{equation}
    \label{eq:stoc-interpolation}
    \mathsf{X}_t = I_t (\mathsf{X}_0, \mathsf{X}_1), \quad \mathsf{X}_0 = a_0 \mathsf{Z} + b_0 \mathsf{X}_1, \quad \mathsf{Z} \sim \gamma_d, \quad \mathsf{X}_1 \sim \nu.
\end{equation}
\end{definition}

\begin{remark}
    It is obvious that the marginal distribution of $\mathsf{X}_t$ satisfies $\mathsf{X}_t \overset{d}{=} a_t \mathsf{Z} + b_t \mathsf{X}_1$ with $\mathsf{Z} \sim \gamma_d, \mathsf{X}_1 \sim \nu$.
\end{remark}

Motivated by the time-varying properties of the Gaussian stochastic interpolation, we derive that its marginal flow satisfies the continuity equation. This result characterizes the dynamics of the marginal density flow of the Gaussian stochastic interpolation.

\begin{theorem} \label{thm:flow-gif}
Suppose that Assumption \ref{assump:well-defined} holds.
Then the marginal flow $(p_t)_{t \in [0, 1]}$ of the Gaussian stochastic interpolation $(\mathsf{X}_t)_{t \in [0, 1]}$ between $\mu$ and $\nu$ satisfies the continuity equation
\begin{equation}
    \label{eq:cont-eq}
    \partial_t p_t + \nabla_x \cdot (p_t v(t, x)) = 0, \quad (t, x) \in [0, 1] \times \sR^d, \quad p_0(x) = \tfrac{\diff \mu}{\diff x}(x), \quad p_1(x) = \tfrac{\diff \nu}{\diff x}(x)
\end{equation}
in the weak sense with the velocity field
\begin{align}
    v(t, x) &:= \E[ \dot{a}_t \mathsf{Z} + \dot{b}_t \mathsf{X}_1 | \mathsf{X}_t = x], \quad t \in (0, 1), \label{eq:vf-expect} \\
    v(0, x) &:= \lim_{t \downarrow 0} v(t, x), ~~~~ v(1, x) := \lim_{t \uparrow 1} v(t, x) \label{eq:vf-boundary}.
\end{align}
\end{theorem}

\begin{remark}
We notice that $x = a_t \E[\mathsf{Z} | \mathsf{X}_t = x] + b_t \E[\mathsf{X}_1 | \mathsf{X}_t = x]$ due to Eq. \paref{eq:stoc-interpolation}.
Then it holds that
\begin{equation}
    \label{eq:vf-expect-single-rmk}
    v(t, x) = \tfrac{\dot{a}_t}{a_t} x + \left( \dot{b}_t - \tfrac{\dot{a}_t}{a_t}b_t \right) \E[\mathsf{X}_1 | \mathsf{X}_t = x], \quad t \in (0, 1).
\end{equation}
We also notice that, according to Tweedie's formula (cf. Lemma \ref{lm:tw-formula} in the Appendix), it holds that
\begin{equation}
    \label{eq:tw-formula-rmk}
    s(t, x) = \tfrac{b_t}{a_t^2} \E \left[ \mathsf{X}_1 | \mathsf{X}_t = x \right] - \tfrac{1}{a_t^2} x, \quad t \in (0, 1),
\end{equation}
where $s(t, x)$ is the score function of the marginal distribution of $\mathsf{X}_t \sim p_t$.

Combining \paref{eq:vf-expect-single-rmk} and \paref{eq:tw-formula-rmk}, it follows that the velocity field is a gradient field and its nonlinear term is the score function $s(t, x)$, namely, for any $t \in (0, 1)$,
\begin{align}
    \label{eq:gene-vf-interpolation-rmk}
    v(t, x) = \tfrac{\dot{b}_t}{b_t} x + \left( \tfrac{\dot{b}_t}{b_t} a_t^2 - \dot{a}_t a_t \right) s(t, x).
\end{align}
\end{remark}

\begin{remark}
A relevant result has been provided in the proof of \cite[Proposition 4]{albergo2023building} in a restricted case that $a_0 = 1, b_0 = 0$.
In this case, if $\dot{a}_0, \dot{a}_1, \dot{b}_0, \dot{b}_1$ are well-defined, the velocity field reads
\begin{align*}
    v(0, x) = \dot{a}_0 x + \dot{b}_0 \E_{\nu} [\mathsf{X}_1], \quad
    v(1, x) = \dot{b}_1 x + \dot{a}_1 \E_{\gamma_d} [\mathsf{Z}]
\end{align*}
at time $0$ and $1$.
Otherwise, if any one of $\dot{a}_0, \dot{a}_1, \dot{b}_0, \dot{b}_1$ is not well-defined, the velocity field $v(0, x)$ or $v(1, x)$ should be considered on a case-by-case basis.
In addition, we provide an alternative viewpoint of the relationship between the velocity field associated with stochastic interpolation and the score function of its marginal flow
using Tweedie's formula in Lemma \ref{lm:tw-formula}.
\end{remark}

\begin{remark} [Diffusion process]
The marginal flow of the Gaussian stochastic interpolation \paref{eq:stoc-interpolation} coincides with the time-reversed marginal flow of a diffusion process $(\overline{X}_t)_{t \in [0, 1)}$ \cite[Theorem 3.5]{albergo2023stochastic} defined by
\begin{equation*}
    \diff \overline{X}_t = -\tfrac{\dot{b}_{1-t}}{b_{1-t}} \overline{X}_t + \sqrt{2 \left( \tfrac{\dot{b}_{1-t}}{b_{1-t}} a_{1-t}^2 - \dot{a}_{1-t} a_{1-t} \right)}\diff \overline{W}_t.
\end{equation*}
\end{remark}

\begin{remark} [Gaussian denoising] \label{rm:gauss-channel}
    The Gaussian stochastic interpolation has an information-theoretic interpretation as a time-varying Gaussian channel. Here $a_t^2$ and $b_t^2 / a_t^2$ stand for the noise level and signal-to-noise ratio (SNR) for time $t \in [0, 1]$, respectively. As time $t \to 1$, we are approaching the high-SNR regime, that is, the SNR $b_t^2 / a_t^2$ grows to $\infty$. Moreover, the SNR $b_t^2 / a_t^2$ is monotonically increasing in time $t$ over $[0, 1]$. The Gaussian noise level gets reduced through this Gaussian denoising process.
\end{remark}

\begin{figure}[t!]
\centering
\includegraphics[width=1.5 in, height=1.5 in]{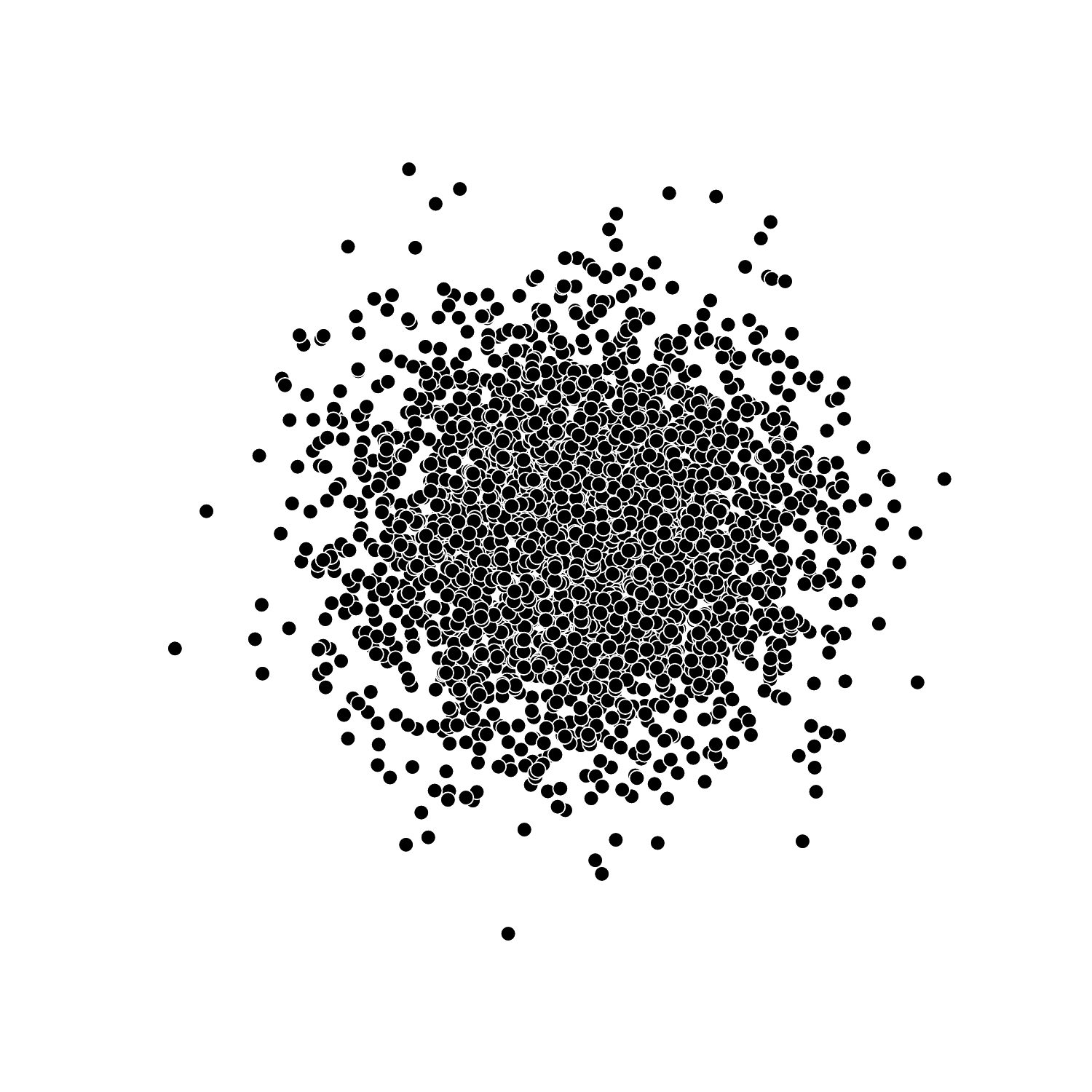}
\includegraphics[width=1.5 in, height=1.5 in]{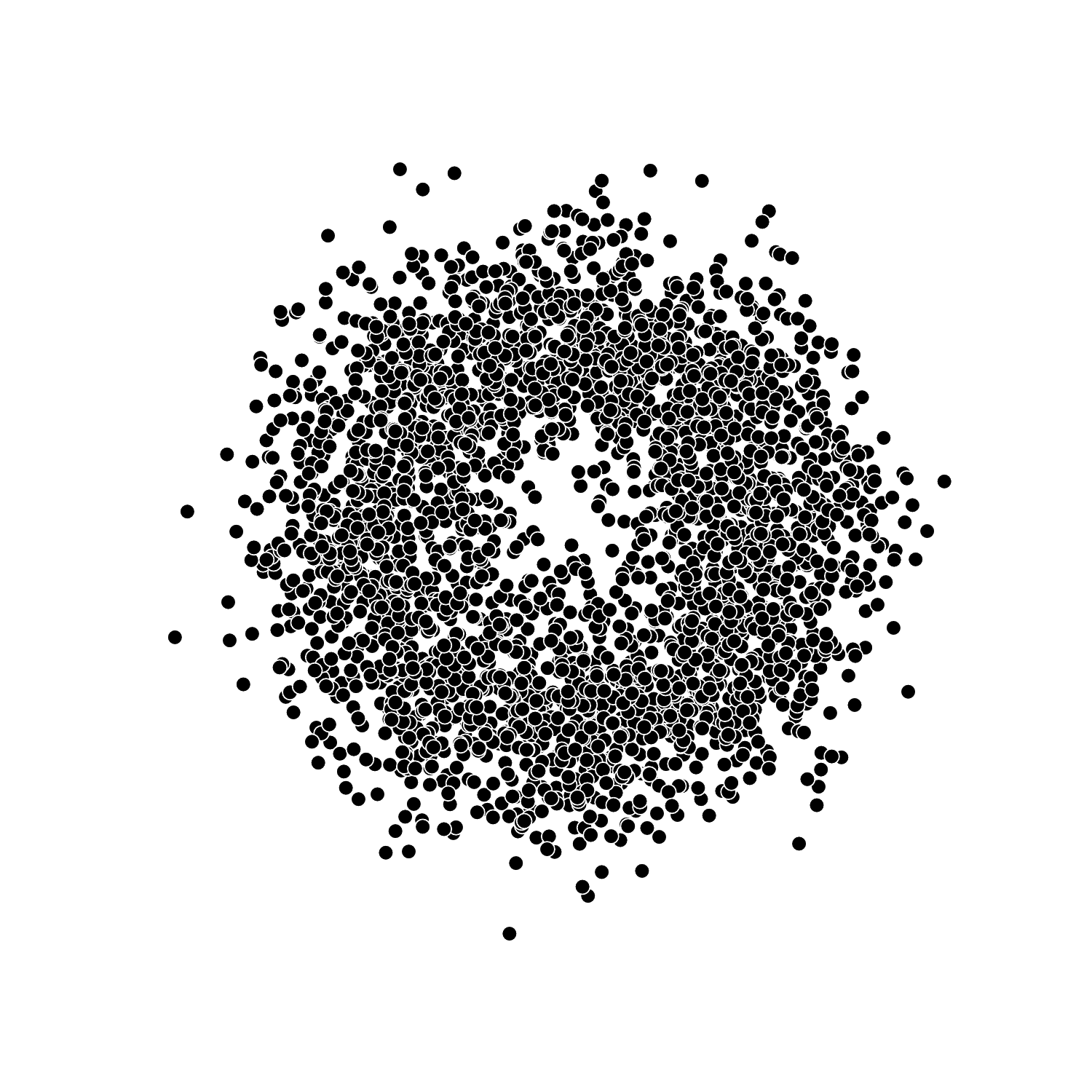}
\includegraphics[width=1.5 in, height=1.5 in]{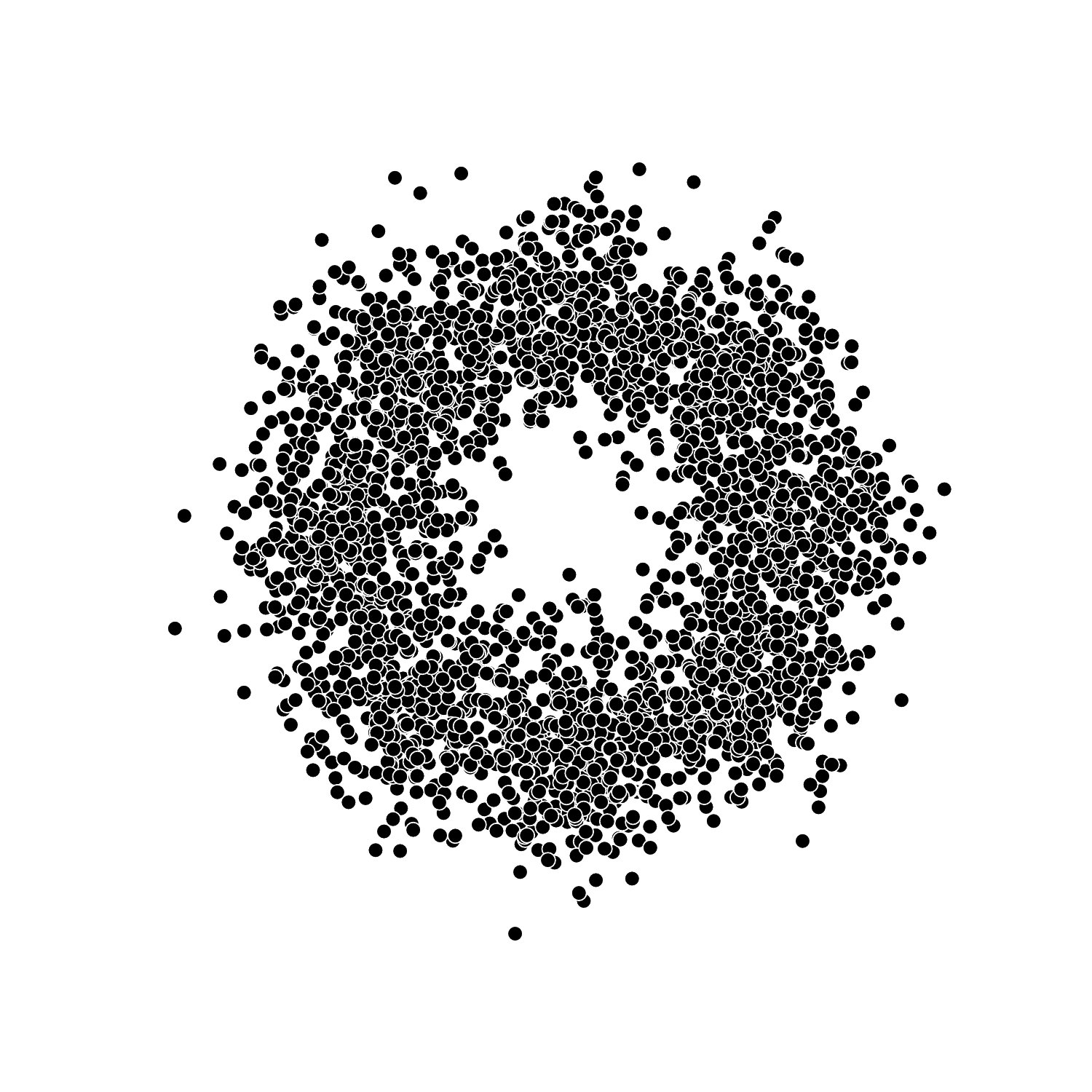}\\
\includegraphics[width=1.5 in, height=1.5 in]{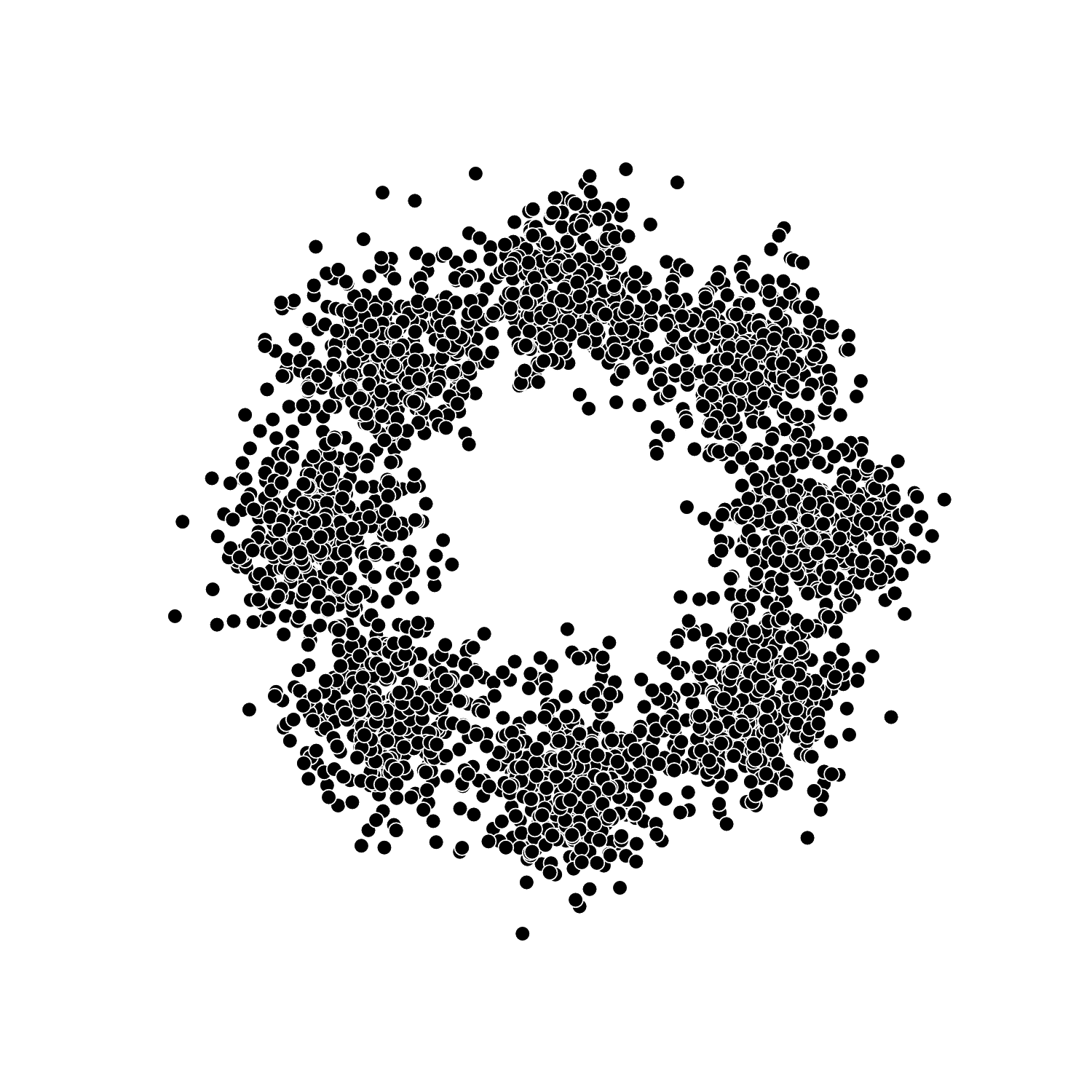}
\includegraphics[width=1.5 in, height=1.5 in]{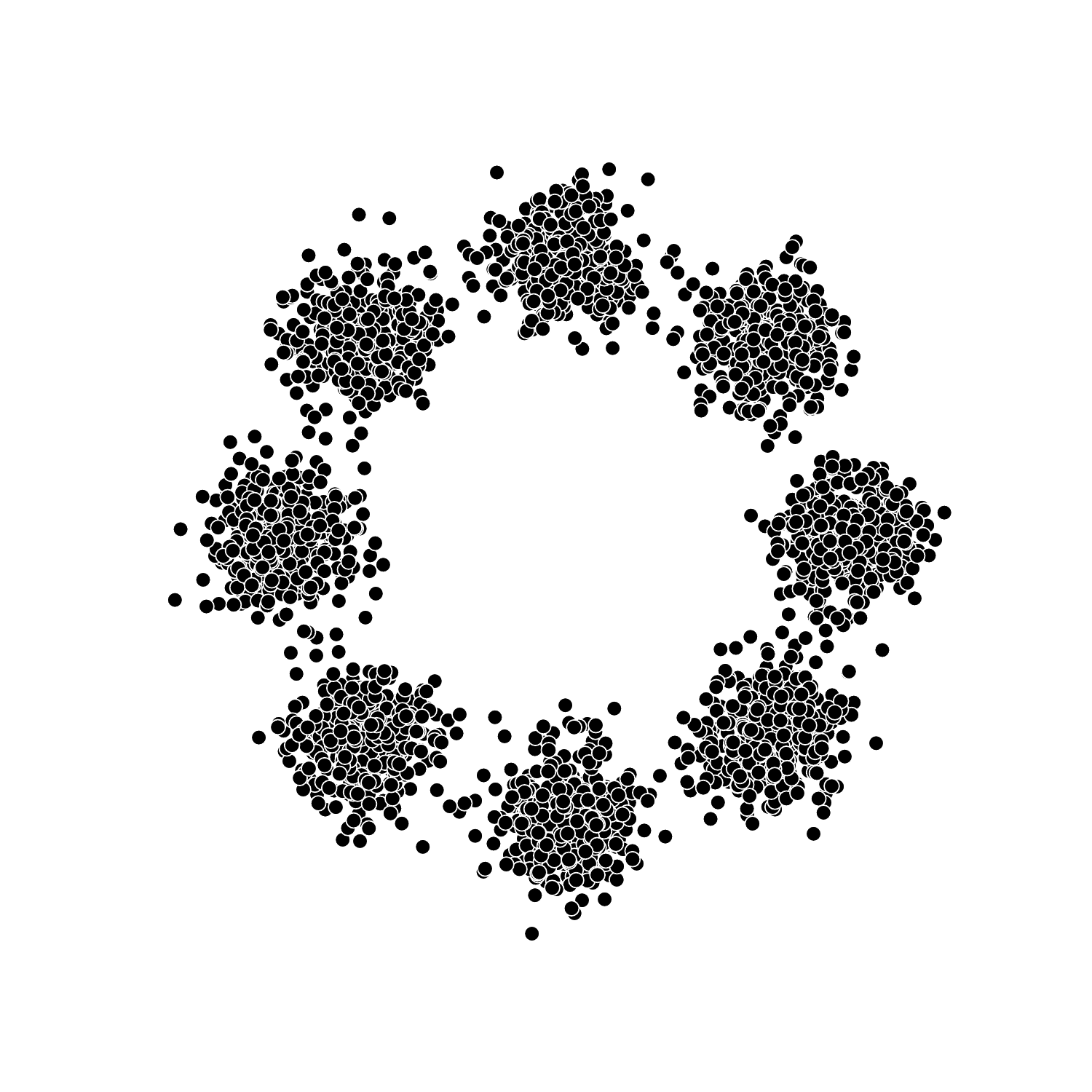}
\includegraphics[width=1.45 in, height=1.45 in]{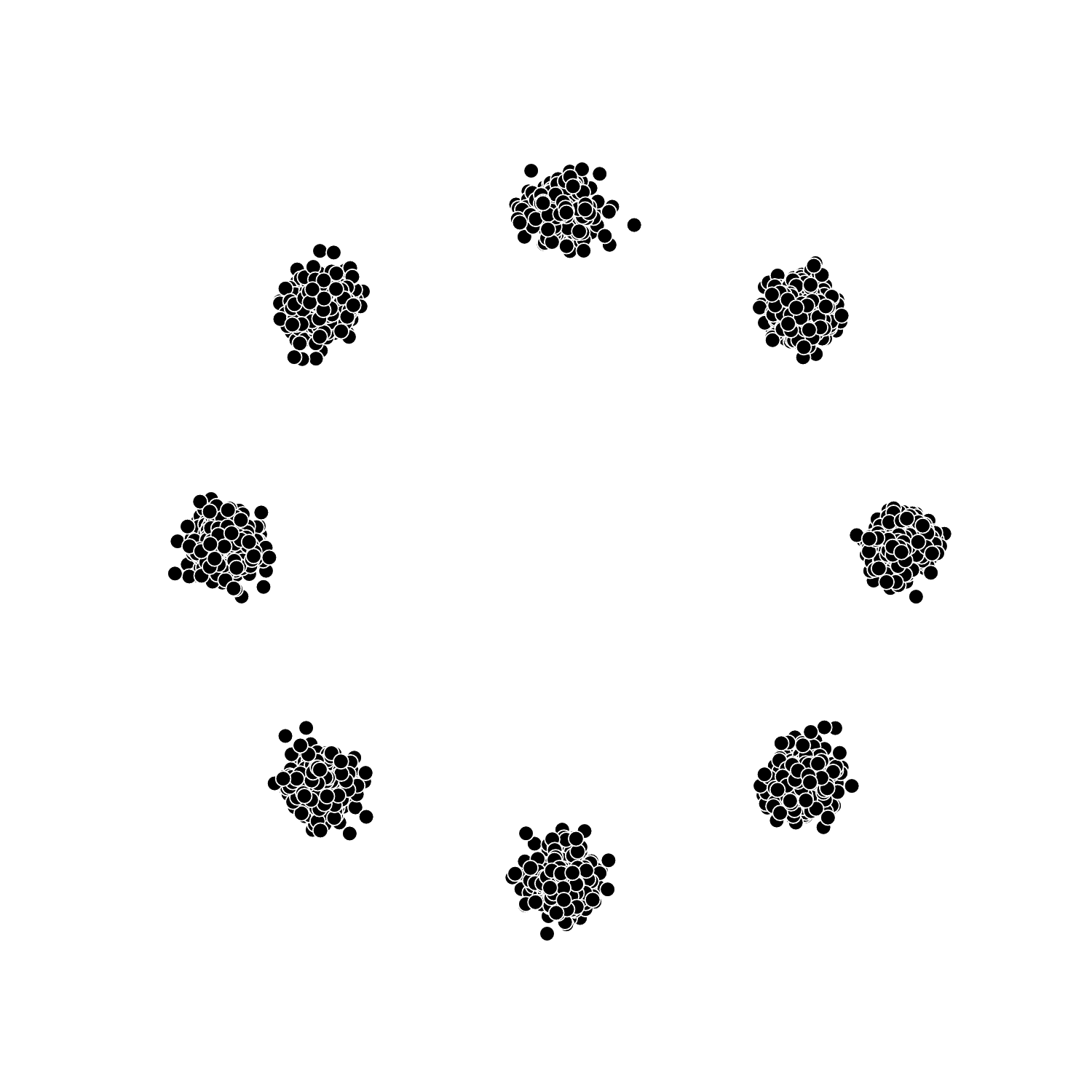}
\caption{Snapshots of a Gaussian interpolation flow based on the F{\"o}llmer interpolant. The source distribution is the standard two-dimensional Gaussian distribution $\gamma_2$, and the target distribution is a mixture of six two-dimensional Gaussian distributions as the shape of a circle. The image panels are placed sequentially from time $t = 0$ to time $t = 1$.}
\label{fig:interp-process}
\end{figure}

We are now ready to define Gaussian interpolation flows by representing the continuity equation \paref{eq:cont-eq} with Lagrangian coordinates \citep{ambrosio2014continuity}.
A basic observation is that GIFs share the same marginal density flow with Gaussian stochastic interpolations.
The continuity equation \paref{eq:cont-eq} plays a central role in the derandomization procedure from Gaussian stochastic interpolations to GIFs.
We additionally illustrate GIFs using a two-dimensional example as in Figure \ref{fig:interp-process}.

\begin{definition} [Gaussian interpolation flow] \label{def:gif}
    Suppose that probability measure $\nu$ satisfies Assumption \ref{assump:well-defined}.
    If $(X_t)_{t \in [0, 1]}$ solves the initial value problem (IVP)
    \begin{equation} \label{eq:ode-ivp}
    \frac{\diff X_t}{\diff t}(x) = v(t, X_t(x)), \quad X_0(x) \sim \mu, \quad t \in [0, 1],
    \end{equation}
    where $\mu$ is defined in Definition \ref{def:measure-interp} and the velocity field $v$ is given by Eq. \paref{eq:vf-expect} and \paref{eq:vf-boundary},
    we call $(X_t)_{t \in [0, 1]}$ a Gaussian interpolation flow associated with the target measure $\nu$.
\end{definition}

\section{Spatial Lipschitz estimates for the velocity field} \label{sec:spatial-lip}

We have explicated the idea of Gaussian denoising with the procedure of Gaussian stochastic interpolation or a Gaussian channel with increasing SNR w.r.t. time.
By interpreting the process as an ODE flow, we derive the framework of Gaussian interpolation flows.
First and foremost, an intuition is that the regularizing effect of Gaussian denoising would ensure the Lipschitz smoothness of the velocity field.
Since the standard Gaussian distribution is both $1$-semi-log-concave and $1$-semi-log-convex, its convolution with a target distribution will maintain its high regularity as long as the target distribution satisfies the regularity conditions.
We rigorously justify this intuition by establishing spatial Lipschitz estimates for the velocity field.
These estimates are established based on
the upper bounds and lower bounds regarding the Jacobian matrix of the velocity field $v(t, x)$ according to the Cauchy-Lipschitz theorem, which are given in Proposition \ref{prop:vf-bd} below.
To deal with the Jacobian matrix $\nabla_x v(t, x)$, we introduce a covariance expression of it and present the associated upper bounds and lower bounds.

The velocity field $v(t, x)$ is decomposed into a linear term and a nonlinear term, the score function $s(t, x)$.
To analyze the Jacobian $\nabla_x v(t, x)$, we only need to focus on $\nabla_x s(t, x)$, that is, $\nabla^2_x \log p_t(x)$. To ease the notation, we would henceforth use $\mathsf{Y}$ 
for $\mathsf{X}_1$.
Correspondingly, we replace $p_1(x)$ with $p_1(y)$
for the density function of $\mathsf{Y}$.

According to  Bayes' theorem, the marginal density $p_t$ of $\mathsf{X}_t$ satisfies
$$p_t (x) = \int p(t, x | y) p_1(y) \diff y$$
where $\mathsf{Y} \sim p_1(y)$ and $p(t, x | y) = \varphi_{b_t y, a_t^2} (x)$ is a conditional distribution induced by the Gaussian noise.
Due to the factorization $p_t (x) p(y | t, x) = p(t, x | y) p_1(y)$,
the score function $s(t, x)$ and its derivative $\nabla_x s(t, x)$ have the following expressions
\begin{align*}
    s(t, x) = - \nabla_x \log p(y | t, x) - \tfrac{x - b_t y}{a_t^2} , \quad
    \nabla_x s(t, x) = - \nabla^2_x \log p(y | t, x) - \tfrac{1}{a_t^2}\rmI_d.
\end{align*}
Thanks to the expressions above, a covariance matrix expression of $\nabla_x s(t, x)$ is endowed by the exponential family property of $p(y | t, x)$.

\begin{lemma} \label{lm:cond-cov}
The conditional distribution $p(y | t, x)$ is an exponential family distribution and a covariance matrix expression of the log-Hessian matrix $\nabla^2_x \log p(y | t, x)$ for any $t \in (0, 1)$ is given by
\begin{equation}
    \label{eq:cov-cond-score-claim}
    \nabla^2_x \log p(y | t, x) = - \tfrac{b_t^2} {a_t^4} \Cov (\mathsf{Y} | \mathsf{X}_t = x),
\end{equation}
where $\Cov (\mathsf{Y} | \mathsf{X}_t = x)$ is the covariance matrix of $\mathsf{Y} | \mathsf{X}_t = x \sim p(y | t, x)$.
Moreover, for any $t \in (0, 1)$, it holds that
\begin{equation}
    \label{eq:cov-gene-score-claim}
    \nabla_x s(t, x) = \tfrac{b_t^2} {a_t^4} \Cov (\mathsf{Y} | \mathsf{X}_t = x) - \tfrac{1}{a_t^2} \rmI_d,
\end{equation}
and that
\begin{equation}
    \label{eq:cov-gene-vf}
    \nabla_x v(t, x) = \tfrac{b_t^2}{a_t^2} \left( \tfrac{\dot{b}_t}{b_t} - \tfrac{\dot{a}_t}{a_t} \right) \Cov (\mathsf{Y} | \mathsf{X}_t = x) + \tfrac{\dot{a}_t} {a_t} \rmI_d.
\end{equation}
\end{lemma}

\begin{remark}
Since $\partial_t \left( \tfrac{b_t^2}{a_t^2} \right) = \tfrac{2 b_t^2}{a_t^2} \left( \tfrac{\dot{b}_t}{b_t} - \tfrac{\dot{a}_t}{a_t} \right),$
it follows from (\ref{eq:cov-gene-vf}) that the derivative of the SNR with respect to  time $t$ controls the dependence of $\nabla_x v(t, x)$ on $\Cov (\mathsf{Y} | \mathsf{X}_t = x)$.
\end{remark}

The representation \paref{eq:cov-gene-vf} can be used to upper bound and lower bound $\nabla_x v(t, x)$. This technique has been widely used to deduce the regularity of the score function concerning the space variable \citep{mikulincer2021brownian, mikulincer2023lipschitz, chen2023sampling, lee2023convergence, chen2023improved}. The covariance matrix expression \paref{eq:cov-gene-score-claim} of the score function has a close connection with the Hatsell-Nolte identity in information theory \citep{hatsell1971some, palomar2005gradient, wu2011functional, cai2014optimal, wibisono2017information, wibisono2018convexity-heat, wibisono2018convexity-ou, dytso2023meta, dytso2023conditional}.

Employing the covariance expression in Lemma \ref{lm:cond-cov}, we establish several bounds on $\nabla_x v(t, x)$ in the following proposition.

\begin{proposition} \label{prop:vf-bd}
Let $\nu(\diff y) = p_1(y) \diff y$ be a probability measure on $\sR^d$ with\\
 $D := (1/\sqrt{2}) \mathrm{diam} (\mathrm{supp}(\nu))$.
\begin{itemize}
    \item[(a)] For any $t\in (0,1)$,
    \begin{equation*}
        \frac{\dot{a}_t}{a_t} \rmI_d \preceq \nabla_x v(t, x) \preceq \left\{ \frac{b_t (a_t \dot{b}_t - \dot{a}_t b_t)}{a_t^3} D^2 + \frac{\dot{a}_t} {a_t} \right\} \rmI_d.
    \end{equation*}

    \item[(b)] Suppose that $p_1$ is $\beta$-semi-log-convex with $\beta > 0$ and $\mathrm{supp}(p_1) = \sR^d$. Then for any $t\in (0, 1]$,
    \begin{equation*}
        \nabla_x v(t, x) \succeq \frac{\beta a_t \dot{a}_t + b_t \dot{b}_t} {\beta a_t^2 + b_t^2} \rmI_d.
    \end{equation*}

    \item[(c)] Suppose that $p_1$ is $\kappa$-semi-log-concave with $\kappa \in \sR$. Then for any $t \in (t_0, 1]$,
    \begin{align*}
        \nabla_x v(t, x) \preceq \frac{\kappa a_t \dot{a}_t + b_t \dot{b}_t} {\kappa a_t^2 + b_t^2} \rmI_d,
    \end{align*}
    where $t_0$ is the root of the equation $\kappa + \frac{b_t^2}{a_t^2} = 0$ over $t \in (0, 1)$ if $\kappa < 0$ and $t_0 = 0$ if $\kappa \ge 0$.

    \item[(d)] Fix a probability measure $\rho$ on $\sR^d$ supported on a Euclidean ball of radius $R$, and let $\nu := \gamma_{d, \sigma^2} * \rho$ with $\sigma > 0$. Then for any $t \in (0, 1)$,
    \begin{equation*}
        \frac{\dot{a}_t a_t + \sigma^2 \dot{b}_t b_t}{a_t^2 + \sigma^2 b_t^2} \rmI_d
        \preceq \nabla_x v(t, x) \preceq
        \left\{ \frac{a_t b_t (a_t \dot{b}_t - \dot{a}_t b_t)}{(a_t^2 + \sigma^2 b_t^2)^2} R^2 + \frac{\dot{a}_t a_t + \sigma^2 \dot{b}_t b_t}{a_t^2 + \sigma^2 b_t^2} \right\} \rmI_d.
    \end{equation*}

    \item[(e)] Suppose that $\frac{\diff \nu}{\diff \gamma_d}(x)$ is $L$-log-Lipschitz for some $L \ge 0$. Then for any $t \in (0, 1)$,
    \begin{align*}
        & \left\{ \left( \tfrac{\dot{b}_t}{b_t} a_t^2 - \dot{a}_t a_t \right) \left( -B_t - L^2 \left( \tfrac{b_t}{a_t^2 + b_t^2} \right)^2 \right) + \tfrac{\dot{a}_t a_t + \dot{b}_t b_t}{a_t^2 + b_t^2} \right\} \rmI_d \\
        & \preceq \nabla_x v(t, x)
          \preceq \left\{ \left( \tfrac{\dot{b}_t}{b_t} a_t^2 - \dot{a}_t a_t \right) B_t + \tfrac{\dot{a}_t a_t + \dot{b}_t b_t}{a_t^2 + b_t^2} \right\} \rmI_d,
    \end{align*}
    where $B_t := 5L b_t (a_t^2 + b_t^2)^{-\frac32} (L + (\log (\sqrt{a_t^2 + b_t^2}/b_t))^{-\frac12})$.
\end{itemize}
\end{proposition}

Comparing part (a) with part (d) in Proposition \ref{prop:vf-bd}, we can see that the bounds in (a) are consistent with those in (d) in the sense that (a) is a limiting case of part (d) as $\sigma \to 0$. The lower bound in part (a) blows up at time $t = 1$ owing to $a_1 = 0,$ while in part (d) it behaves well since the lower bound in part (d) coincides with a lower bound indicated by the $\tfrac{1}{\sigma^2}$-semi-log-convex property. It reveals that the regularity of the velocity field $v(t, x)$ with respect to the space variable $x$ improves when the target random variable is bounded and is subject to Gaussian perturbation.

The lower bound in part (b) and the upper bound in part (c) are tight in the sense that both of them are attainable for a Gaussian target distribution, that is,
\begin{equation*}
    \nabla_x v(t, x) = \frac{\beta a_t \dot{a}_t + b_t \dot{b}_t} {\beta a_t^2 + b_t^2} \rmI_d \quad \text{if $\nu = \gamma_{d, 1/\beta}$.}
\end{equation*}

The upper and lower bounds in Proposition \ref{prop:vf-bd}-(a) and (e) become vacuous as they both blow up at time $t = 1$. The intuition behind is that the Jacobian matrix of the velocity field can be both lower and upper bounded at time $t = 1$ only if the score function of the target measure is Lipschitz continuous in the space variable $x$. Under an additional Lipschitz score assumption (equivalently, $\beta$-semi-log-convex and $\kappa$-semi-log-concave for some $\beta = -\kappa \ge 0$), the upper and lower bounds in part (a) and part (e) can be strengthened at time $t = 1$ based on the lower bound in (b) and the upper bound in part (c).

According to Proposition \ref{prop:vf-bd}-(a) and (c), there are two upper bounds available that shall be compared with each other. One is the $D^2$-based bound in part (a), and the other is the $\kappa$-based bound in part (c). According to the proof of Proposition \ref{prop:vf-bd} given in the Appendix, these two upper bounds are equal if and only if the corresponding upper bounds on $\Cov (\mathsf{Y} | \mathsf{X}_t = x)$ are equal, that is,
\begin{equation}
    \label{eq:equal-bd}
    D^2 = \left( \kappa + \frac{b_t^2}{a_t^2} \right)^{-1}.
\end{equation}
Then the critical case is $\kappa D^2 = 1$ since simplifying Eq. \paref{eq:equal-bd} reveals that
\begin{equation}
    \label{eq:critical-cond}
    D^{-2} - \kappa = \frac{b_t^2}{a_t^2}.
\end{equation}
We note that $b_t^2 / a_t^2$, ranging over $(0, \infty)$, is monotonically increasing
w.r.t. $t \in (0, 1)$. Suppose that $\kappa D^2 > 1$. Then \paref{eq:critical-cond} has no root over
$t \in (0, 1), $ which implies that the $\kappa$-based bound is tighter over $[0, 1)$, i.e.,
\begin{equation*}
    D^2 > \left( \kappa + \frac{b_t^2}{a_t^2} \right)^{-1}, \quad \forall t \in [0, 1).
\end{equation*}
Otherwise, suppose that $\kappa D^2 < 1$. Then \paref{eq:critical-cond} has a root
$t_1 \in (0, 1),$
which implies that the $D^2$-based bound is tighter over $[0, t_1)$, i.e.,
\begin{align*}
    D^2 < \left( \kappa + \frac{b_t^2}{a_t^2} \right)^{-1}, \quad \forall t \in [0, t_1),
\end{align*}
and that the $\kappa$-based bound is tighter over $[t_1, 1)$, i.e.,
\begin{align*}
    D^2 \ge \left( \kappa + \frac{b_t^2}{a_t^2} \right)^{-1}, \quad \forall t \in [t_1, 1).
\end{align*}

Next, we present several upper bounds on the maximum eigenvalue of the Jacobian matrix of the velocity field $\lambda_{\max} (\nabla_x  v(t, x))$ and its exponential estimates for studying the Lipschitz regularity of the flow maps as noted in
Lemma \ref{lm:flow-map-Lip-bd}.

\begin{corollary} \label{cor:pos-kappa-bd}
Let $\nu$ be a probability measure on $\sR^d$ with $D :=
(1/\sqrt{2}) \mathrm{diam} ( \mathrm{supp}(\nu))$ and suppose that $\nu$ is $\kappa$-semi-log-concave with $\kappa \ge 0$.
\begin{itemize}
    \item[(a)] If $\kappa D^2 \ge 1$, then
        \begin{equation}
            \label{eq:max-egv-ubd-ge1}
            \lambda_{\max} (\nabla_x v(t, x)) \le \theta_t := \frac{\kappa a_t \dot{a}_t + b_t \dot{b}_t} {\kappa a_t^2 + b_t^2}, \ t \in [0, 1].
        \end{equation}

    \item[(b)] If $\kappa D^2 < 1$, then
        \begin{equation}
            \label{eq:max-egv-ubd-less1}
            \lambda_{\max} (\nabla_x v(t, x)) \le \theta_t :=
                \begin{cases}
                    \frac{b_t^2}{a_t^2} \left( \frac{\dot{b}_t}{b_t} - \frac{\dot{a}_t}{a_t} \right) D^2 + \frac{\dot{a}_t} {a_t}, \ &t \in [0, t_1), \\
                    \frac{\kappa a_t \dot{a}_t + b_t \dot{b}_t} {\kappa a_t^2 + b_t^2},  \ &t \in [t_1, 1],
                \end{cases}
        \end{equation}
        where $t_1$ solves  \paref{eq:critical-cond}.
\end{itemize}
\end{corollary}

\begin{corollary} \label{cor:neg-kappa-bd}
Let $\nu$ be a probability measure on $\sR^d$ with $D :=  (1/\sqrt{2}) \mathrm{diam} (\mathrm{supp}(\nu)) < \infty$ and suppose that $\nu$ is $\kappa$-semi-log-concave with $\kappa < 0$.
Then
\begin{equation}
    \label{eq:max-egv-ubd-kappa-nega}
    \lambda_{\max} (\nabla_x v(t, x)) \le \theta_t :=
    \begin{cases}
        \frac{b_t^2}{a_t^2} \left( \frac{\dot{b}_t}{b_t} - \frac{\dot{a}_t}{a_t} \right) D^2 + \frac{\dot{a}_t} {a_t}, \ &t \in [0, t_1), \\
        \frac{\kappa a_t \dot{a}_t + b_t \dot{b}_t} {\kappa a_t^2 + b_t^2},  \ &t \in [t_1, 1],
    \end{cases}
\end{equation}
where $t_1$ solves \paref{eq:critical-cond}.
\end{corollary}

\begin{corollary} \label{cor:mog-bd}
Fix a probability measure $\rho$ on $\sR^d$ supported on a Euclidean ball of radius $R$ and let $\nu := \gamma_{d, \sigma^2} * \rho$ with $\sigma > 0$.
Then
\begin{equation}
    \label{eq:max-egv-mog}
    \lambda_{\max} (\nabla_x v(t, x)) \le \theta_t :=
    \frac{\dot{a}_t a_t + \sigma^2 \dot{b}_t b_t}{a_t^2 + \sigma^2 b_t^2} + \frac{a_t b_t (a_t \dot{b}_t - \dot{a}_t b_t)}{(a_t^2 + \sigma^2 b_t^2)^2} R^2.
\end{equation}
\end{corollary}

\begin{corollary} \label{cor:log-lip}
Suppose that $\nu$ is $\kappa$-semi-log-concave for some $\kappa \le 0$, and $\frac{\diff \nu}{\diff \gamma_d}(x)$ is $L$-log-Lipschitz for some $L \ge 0$. Then
\begin{equation}
    \label{eq:max-egv-ubd-log-lip}
    \lambda_{\max} (\nabla_x v(t, x)) \le \theta_t :=
    \begin{cases}
        \left( \tfrac{\dot{b}_t}{b_t} a_t^2 - \dot{a}_t a_t \right) B_t + \tfrac{\dot{a}_t a_t + \dot{b}_t b_t}{a_t^2 + b_t^2}, \ &t \in [0, t_2), \\
        \frac{\kappa a_t \dot{a}_t + b_t \dot{b}_t} {\kappa a_t^2 + b_t^2},  \ &t \in [t_2, 1],
    \end{cases}
\end{equation}
where $B_t := 5L b_t (a_t^2 + b_t^2)^{-\frac32} (L + (\log (\sqrt{a_t^2 + b_t^2}/b_t))^{-\frac12})$ and $t_2 \in (t_0, 1)$.
\end{corollary}

\section{Well-posedness and Lipschtiz flow maps} \label{sec:well-posed}

In this section, we study the well-posedness of GIFs and the Lipschitz properties of their flow maps.
We also show that the marginal distributions of GIFs satisfy the log-Sobolev inequality and the Poincar{\'e} inequality if Assumptions \ref{assump:well-defined} and \ref{assump:geom-prop} are satisfied.

\begin{theorem}[Well-posedness] \label{thm:well-posed}
Suppose Assumptions \ref{assump:well-defined} and \ref{assump:geom-prop}-(i), (iii), or (iv) are satisfied. Then there exists a unique solution $(X_t)_{t \in [0, 1]}$ to the
IVP \paref{eq:ode-ivp}. Moreover, the push-forward measure satisfies ${X_t}_{\#} \mu = \mathrm{Law}(a_t \mathsf{Z} + b_t \mathsf{X}_1)$ with $\mathsf{Z} \sim \gamma_d, \mathsf{X}_1 \sim \nu$.
\end{theorem}

\begin{theorem} \label{thm:well-posed-bounded}
Suppose Assumptions \ref{assump:well-defined} and \ref{assump:geom-prop}-(ii) are satisfied.
For any $\underline{t} \in (0, 1)$, there exists a unique solution $(X_t)_{t \in [0, 1-\underline{t}]}$ to the IVP \paref{eq:ode-ivp}. Moreover, the push-forward measure satisfies ${X_t}_{\#} \mu = \mathrm{Law}(a_t \mathsf{Z} + b_t \mathsf{X}_1)$ with $\mathsf{Z} \sim \gamma_d, \mathsf{X}_1 \sim \nu$.
\end{theorem}

\begin{corollary} [Time-reversed flow] \label{cor:time-reve-flow}
Suppose Assumptions \ref{assump:well-defined} and \ref{assump:geom-prop}-(i), (iii), or (iv) are satisfied. Then the time-reversed flow $(X^{*}_t)_{t \in [0, 1]}$ associated with $\nu$ is a unique solution to the IVP:
\begin{equation}
    \label{eq:ode-ivp-reve}
    \frac{\diff X^{*}_t}{\diff t}(x) = -v(1-t, X^{*}_t(x)), \quad X^{*}_0(x) \sim \nu, \quad t \in [0, 1].
\end{equation}
The push-forward measure satisfies ${X^{*}_t}_{\#} \nu = \mathrm{Law}(a_{1-t} \mathsf{Z} + b_{1-t} \mathsf{X}_1)$ where $\mathsf{Z} \sim \gamma_d, \mathsf{X}_1 \sim \nu$.
Moreover, the flow map satisfies $X^{*}_t(x) = X_t^{-1}(x)$.
\end{corollary}

\begin{corollary} \label{cor:time-reve-flow-bounded}
Suppose Assumptions \ref{assump:well-defined} and \ref{assump:geom-prop}-(ii) are satisfied.
For any $\underline{t} \in (0, 1)$, the time-reversed flow $(X^{*}_t)_{t \in [\underline{t}, 1]}$ associated with $\nu$ is a unique solution to the IVP:
\begin{equation}
    \label{eq:ode-ivp-reve-bounded}
    \frac{\diff X^{*}_t}{\diff t}(x) = -v(1-t, X^{*}_t(x)), \quad X^{*}_{\underline{t}}(x) \sim \mathrm{Law}(a_{1-\underline{t}} \mathsf{Z} + b_{1-\underline{t}} \mathsf{X}_1), \quad t \in [\underline{t}, 1],
\end{equation}
where $\mathsf{Z} \sim \gamma_d, \mathsf{X}_1 \sim \nu$.
The push-forward measure satisfies ${X^{*}_t}_{\#} \nu = \mathrm{Law}(a_{1-t} \mathsf{Z} + b_{1-t} \mathsf{X}_1)$.
Moreover, the flow map satisfies $X^{*}_t(x) = X_t^{-1}(x)$.
\end{corollary}

Based on the well-posedness of the flow, we can provide an upper bound on the Lipschitz constant of the induced flow map.

\begin{lemma} \label{lm:flow-map-Lip-bd}
Suppose that a flow $(X_t)_{t \in [0, 1]}$ is well-posed with a velocity field $v(t, x): [0, 1] \times \sR^d \to \sR^d$ of class $C^1$ in $x$, and that for any $(t, x) \in [0, 1] \times \sR^d$, it holds $\nabla_x v(t, x) \preceq \theta_t \rmI_d$.
Let the flow map $X_{s,t}: \sR^d \to \sR^d$ be of class $C^1$ in $x$ for any $0 \le s \le t \le 1$.
Then the flow map $X_{s,t}$ is Lipschitz continuous with an upper bound of its Lipschitz constant given by
\begin{equation}
    \Vert \nabla_x X_{s,t}(x) \Vert_{2,2} \le \exp \left( \int_s^t \theta_u \diff u \right).
\end{equation}
\end{lemma}

Using Lemma \ref{lm:flow-map-Lip-bd}, we show that the flow map of a GIF is Lipschitz continuous in the space variable $x$.

\begin{proposition}[Lipschitz mappings] \label{prop:lip-map}
 Suppose that Assumptions \ref{assump:well-defined} and \ref{assump:geom-prop}-(i) hold.
\begin{itemize}
    \item[(i)] If $\nu$ is $\kappa$-semi-log-concave for some $\kappa > 0$, then the flow map $X_1(x)$ is a Lipschitz mapping, that is,
    \begin{equation*}
\Vert \nabla_x X_1(x) \Vert_{2,2} \le \frac{1}{\sqrt{\kappa a_0^2 + b_0^2}}, \quad \forall x \in \sR^d.
    \end{equation*}
    In particular, if $a_0 = 1$ and $b_0 = 0$, then
    \begin{equation*}
        \Vert \nabla_x X_1(x) \Vert_{2,2} \le \frac{1} {\sqrt{\kappa}}, \quad \forall x \in \sR^d.
    \end{equation*}

    \item[(ii)] If $\nu$ is $\beta$-semi-log-convex for some $\beta >0$, then the time-reversed flow map $X^{*}_1(x)$ is a Lipschitz mapping, that is,
    \begin{equation*}
        \Vert \nabla_x X^{*}_1(x) \Vert_{2,2} \le \sqrt{\beta a_0^2 + b_0^2}, \quad \forall x \in \mathrm{supp}(\nu).
    \end{equation*}
    In particular, if $a_0 = 1$ and $b_0 = 0$, then
    \begin{equation*}
        \Vert \nabla_x X^{*}_1(x) \Vert_{2,2} \le \sqrt{\beta}, \quad \forall x \in \mathrm{supp}(\nu).
    \end{equation*}
\end{itemize}
\end{proposition}

\begin{proposition}[Gaussian mixtures] \label{prop:lip-map-mog}
Suppose that Assumptions \ref{assump:well-defined} and \ref{assump:geom-prop}-(iii) hold.
Then the flow map $X_1(x)$ is a Lipschitz mapping, that is,
\begin{equation*}
    \Vert \nabla_x X_1(x) \Vert_{2,2} \leq \frac{\sigma}{\sqrt{a_0^2 + \sigma^2 b_0^2}} \exp \left( \frac{a_0^2}{a_0^2 + \sigma^2 b_0^2} \cdot \frac{R^2}{2 \sigma^2} \right), \quad \forall x \in \sR^d.
\end{equation*}
In particular, if $a_0 = 1$ and $b_0 = 0$, then
\begin{equation*}
    \Vert \nabla_x X_1(x) \Vert_{2,2} \leq \sigma \exp \left( \frac{R^2}{2 \sigma^2} \right), \quad \forall x \in \sR^d.
\end{equation*}
Moreover, the time-reversed flow map $X^{*}_1(x)$ is a Lipschitz mapping, that is,
\begin{equation*}
    \Vert \nabla_x X^{*}_1(x) \Vert_{2,2} \le \sqrt{\sigma^{-2} a_0^2 + b_0^2}, \quad \forall x \in \mathrm{supp}(\nu).
\end{equation*}
In particular, if $a_0 = 1$ and $b_0 = 0$, then
\begin{equation*}
    \Vert \nabla_x X^{*}_1(x) \Vert_{2,2} \le \frac{1}{\sigma}, \quad \forall x \in \mathrm{supp}(\nu).
\end{equation*}
\end{proposition}

\begin{remark}
    Well-posed GIFs produce diffeomorphisms that transport the source measure onto the target measure. The diffeomorphism property of the transport maps are relevant to the
    auto-encoding and cycle consistency properties of their generative modeling applications. We defer a detailed discussion to Section \ref{sec:app}.
\end{remark}

    Early stopping implicitly mollifies the target measure with a small Gaussian noise. For image generation tasks (with bounded pixel values), the mollified target measure is indeed a Gaussian mixture distribution considered in Theorem \ref{prop:lip-map-mog}. The regularity of the target measure largely gets enhanced through such mollification, especially when the target measure is supported on a low-dimensional manifold in accordance with the data manifold hypothesis. Therefore, although such a diffeomorphism $X_1(x)$ may not be well-defined for general bounded target measures, an off-the-shelf solution would be to perturb the target measure with a small Gaussian noise or to employ the early stopping technique. Both approaches will smooth the landscape of the target measure.

\begin{proposition} \label{prop:func-ineq}
    Suppose the target measure $\nu$ satisfies the log-Sobolev inequality with constant $C_{\mathrm{LS}}(\nu)$. Then the marginal distribution of the
    GIF $(p_t)_{t \in [0, 1]}$ satisfies the log-Sobolev inequality, and its log-Sobolev constant $C_{\mathrm{LS}}(p_t)$ is bounded as
    \begin{align*}
        C_{\mathrm{LS}}(p_t) \le a_t^2 + b_t^2 C_{\mathrm{LS}}(\nu).
    \end{align*}
    Moreover, suppose the target measure $\nu$ satisfies the Poincar{\'e} inequality with constant $C_{P}(\nu)$. Then the marginal distribution of the GIF $(p_t)_{t \in [0, 1]}$ satisfies the Poincar{\'e} inequality, and its Poincar{\'e} constant $C_{\mathrm{P}}(p_t)$ is bounded as
    \begin{align*}
        C_{\mathrm{P}}(p_t) \le a_t^2 + b_t^2 C_{\mathrm{P}}(\nu).
    \end{align*}
\end{proposition}

    The log-Sobolev and Poincar{\'e} inequalities (see Definitions \ref{def:lsi} and \ref{def:pi}) are fundamental tools
    for establishing convergence guarantees for Langevin Monte Carlo algorithms. From an algorithmic viewpoint, the predictor-corrector algorithm in score-based diffusion models and the corresponding probability flow ODEs essentially combine the ODE numerical solver (performing as the predictor) and the overdamped Langevin diffusion (performing as the corrector) to simulate samples from the marginal distributions \citep{song2021scorebased}. Proposition \ref{prop:func-ineq} shows that the marginal distributions all satisfy the log-Sobolev and Poincar{\'e} inequalities under mild assumptions on the target distribution. This conclusion suggests that Langevin Monte Carlo algorithms are certified to have convergence guarantees for sampling from the marginal distributions of
     GIFs. Furthermore, the target distributions covered in Assumption \ref{assump:geom-prop} are
      shown to satisfy the log-Sobolev and Poincar{\'e} inequalities \citep{mikulincer2021brownian, dai2023lipschitz, fathi2023transportation}, which 
      suggests that the assumptions of Proposition \ref{prop:func-ineq} generally hold.

\section{Applications to generative modeling} \label{sec:app}

Auto-encoding is a primary principle in learning a latent representation with generative models \cite[Chapter 14]{goodfellow2016deep}. Meanwhile, the concept of cycle consistency
is important to unpaired image-to-image translation between the source and target domains \citep{zhu2017unpaired}. The recent work by \citet{su2023dual} propose the dual diffusion implicit bridges (DDIB) for image-to-image translation, which shows a strong pattern of exact auto-encoding and image-to-image translation. DDIBs are built upon the denoising diffusion implicit models (DDIM), which share the same probability flow ODE with VESDE (considered as VE interpolant in Table \ref{tab:vec-interp}), as pointed out by \cite[Proposition 1]{song2021denoising}.
First, DDIBs attain latent embeddings of source images encoded with one DDIM operating in the source domain. The encoding embeddings are then decoded using another DDIM trained in the target domain to construct target images. The whole process consisting of two DDIMs seems to be cycle consistent up to numerical errors. Several phenomena of auto-encoding and cycle consistency are observed in the unpaired data generation procedure with DDIBs.

We replicate the 2D experiments by \citet{su2023dual} in
Figures \ref{fig:auto-encoding} and \ref{fig:cycle-consistency}
to show the phenomena of approximate auto-encoding and cycle consistency of GIFs\footnote{The implementation is based on the GitHub repository at \url{https://github.com/suxuann/ddib}.}. To elucidate the empirical auto-encoding and cycle consistency for measure transport, we derive Corollaries \ref{cor:ae} and \ref{cor:cycle} below and analyze the transport maps defined by GIFs (covering the probability flow ODE of VESDE used by DDIBs). We consider the continuous-time framework and the population level, which precludes learning errors including the time discretization errors and velocity field estimation errors, and show that the transport maps naturally possess the exact auto-encoding and cycle consistency properties at the population level.

\begin{figure}[t!]
\centering
\includegraphics[width=0.9\textwidth]{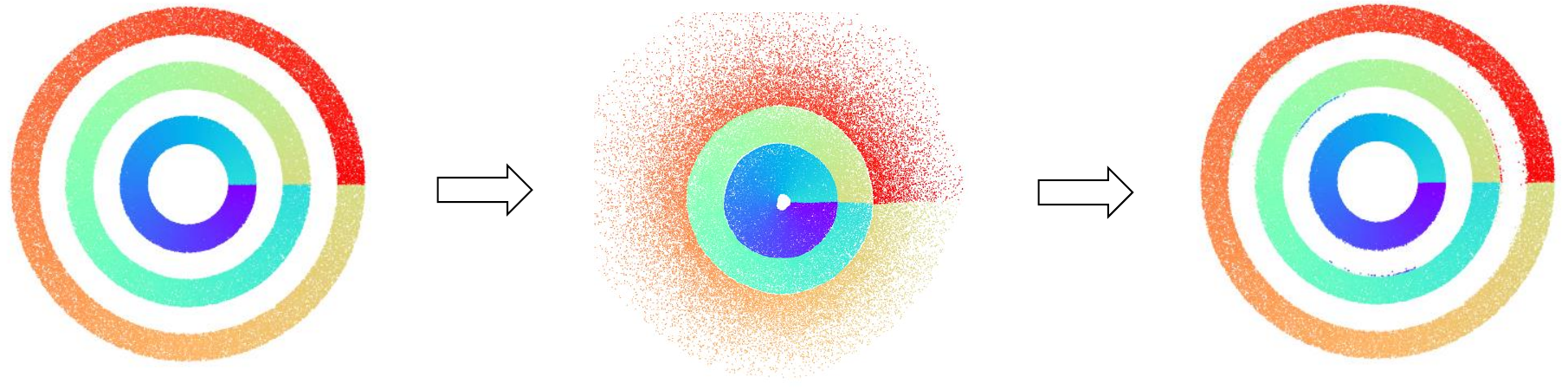}
\caption{An illustration of auto-encoding using DDIBs. The Concentric Rings data in the source domain (the first panel) is encoded into the latent domain (the second panel), and then decoded into the source domain (the third panel).
According to the consistent color pattern and pointwise correspondences across the domains, both the learned encoder mapping and the learned decoder mapping exhibit approximate Lipschitz continuity with respect to the space variable. One justification of such auto-encoding observation is presented in Corollary \ref{cor:ae} where we prove that the composition of the encoder map and the decoder map yields an identity map.}
\label{fig:auto-encoding}
\end{figure}

\begin{corollary}[Auto-encoding] \label{cor:ae}
Suppose Assumptions \ref{assump:well-defined} and \ref{assump:geom-prop}-(i), (iii), or (iv) hold for a target measure $\nu$.
The Gaussian interpolation flow $(X_t)_{t \in [0, 1]}$ and its time-reversed flow $(X_t^*)_{t \in [0, 1]}$ form an auto-encoder with a Lipschitz encoder $X_1^*(x)$ and a Lipschitz decoder $X_1(x)$. The auto-encoding property holds in the sense that
\begin{equation}
    \label{eq:identity-ae}
    X_1 \circ X_1^* = \rmI_d.
\end{equation}
\end{corollary}

\begin{corollary}[Cycle consistency] \label{cor:cycle}
Suppose Assumptions \ref{assump:well-defined} and \ref{assump:geom-prop}-(i), (iii), or (iv) hold for the target measures $\nu_1$ and $\nu_2$.
For the target measure $\nu_1$, we define the Gaussian interpolation flow $(X_{1, t})_{t \in [0, 1]}$ and its time-reversed flow $(X_{1, t}^*)_{t \in [0, 1]}$.
We also define the Gaussian interpolation flow $(X_{2, t})_{t \in [0, 1]}$ and its time-reversed flow $(X_{2, t}^*)_{t \in [0, 1]}$ for the target measure $\nu_2$ using the same $a_t$ and $b_t$.
Then the transport maps $X_{1, 1}(x)$, $X_{1, 1}^*(x)$, $X_{2, 1}(x)$, and $X_{2, 1}^*(x)$ are Lipschitz continuous in the space variable $x$. Furthermore, the cycle consistency property holds in the sense that
\begin{equation}
    \label{eq:identity-cycle}
    X_{1, 1} \circ X_{2, 1}^* \circ X_{2, 1} \circ X_{1, 1}^* = \rmI_d.
\end{equation}
\end{corollary}

Corollaries \ref{cor:ae} and \ref{cor:cycle} show that the auto-encoding and cycle consistency properties hold for the flows at the population level. These results provide insights to the approximate auto-encoding and cycle consistency properties at the sample level.

\begin{figure}[t!]
\centering
\includegraphics[width=0.9\textwidth]{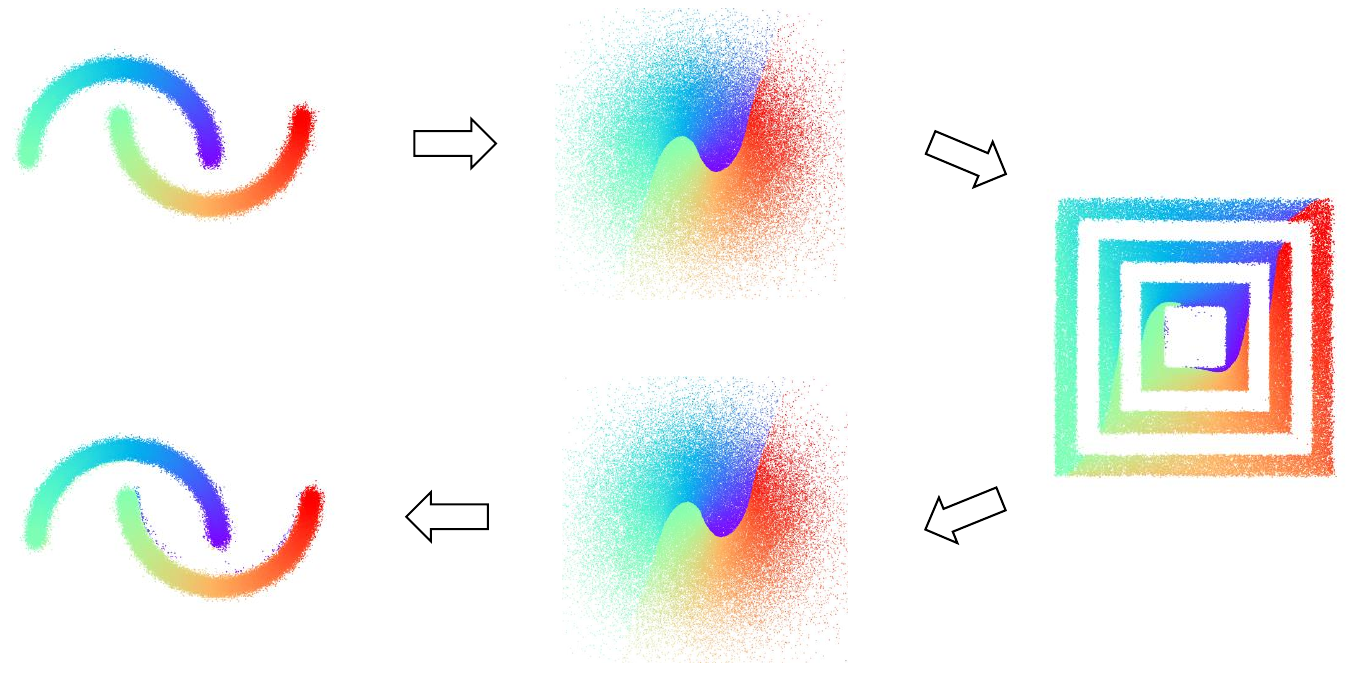}
\caption{An illustration of cycle consistency using DDIBs. The cycle consistency property is manifested through the consistency of color patterns across the transformations. We transform the Moons data in the source domain onto the Concentric Squares data in the target domain, and then complete the cycle by mapping the target data back to the source domain. The latent spaces play a central role in the bidirectional translation. We provide a proof in Corollary \ref{cor:cycle} accounting for the cycle consistency property.}
\label{fig:cycle-consistency}
\end{figure}

There are several types of errors introduced in the training of
GIFs.
On the one hand, the approximation in specifying source measures would exert influence on modeling the distribution.
On the other hand, the approximation in the velocity field
also affects
the distribution learning error.
We
use the stability analysis method in the differential equations theory to address the potential effects of these
errors.

\begin{corollary} \label{cor:prepare-bd}
Suppose Assumptions \ref{assump:well-defined} and \ref{assump:geom-prop}-(i), (iii), or (iv) hold.
It holds that
\begin{align*}
    C_1 := \sup_{x \in \sR^d} \Vert \nabla_x X_1(x) \Vert_{2,2} < \infty, \quad
    C_2 := \sup_{(t, x) \in [0, 1] \times \sR^d} \Vert \nabla_x v(t, x) \Vert_{2,2} < \infty.
\end{align*}
\end{corollary}

\begin{proposition}[Stability in the source distribution] \label{prop:stab-iv}
Suppose Assumptions \ref{assump:well-defined} and \ref{assump:geom-prop}-(i), (iii), or (iv) hold.
If the source measure $\mu = \mathrm{Law}(a_0 \mathsf{Z} + b_0 \mathsf{X}_1)$ is replaced with the Gaussian measure $\gamma_{d, a_0^2}$,
then the stability of the transport map $X_1$ is guaranteed by the $W_2$ distance between the push-forward measure ${X_1}_{\#} \gamma_{d, a_0^2}$ and the target measure $\nu = \mathrm{Law}(\mathsf{X}_1)$ as follows
\begin{equation}
    \label{eq:w2-bd}
    W_2({X_1}_{\#} \gamma_{d, a_0^2}, \nu) \le C_1 b_0 \sqrt{\E_{\nu} [\Vert \mathsf{X_1} \Vert^2]} \exp (C_2 d).
\end{equation}
\end{proposition}

    The stability analysis in Proposition \ref{prop:stab-iv} provides insights
    into the selection of source measures for learning probability flow ODEs and GIFs.
    The error bound \paref{eq:w2-bd} demonstrates that when the signal intensity is reasonably small in the source measure, that is, $b_0 \ll 1$, the distribution estimation error, induced by the approximation with a Gaussian source measure, is small as well in the sense of the quadratic Wasserstein distance.
    Using a Gaussian source measure to replace the true convolution source measure is a common approximation method for learning probability flow ODEs and
    GIFs. Our analysis shows this replacement is
    reasonable for the purpose of distribution estimation.

The Alekseev-Gr{\"o}bner formula and its stochastic variants \citep{moral2022backward} have been shown effective in quantifying the stability of well-posed ODE and SDE flows against perturbations of its velocity field or drift \citep{bortoli2022convergence, benton2023error}. We state these results below for convenience.

\begin{lemma} \cite[Theorem 14.5]{hairer1993classical} \label{lm:ag-formula}
Let $(X_t)_{t \in [0, 1]}$ and $(Y_t)_{t \in [0, 1]}$ solve the following IVPs, respectively
\begin{align*}
    \frac{\diff X_t}{\diff t} &= v(t, X_t), \quad X_0 = x_0, \quad t \in [0, 1], \\
    \frac{\diff Y_t}{\diff t} &= \tilde{v}(t, Y_t), \quad~ Y_0 = x_0, \quad t \in [0, 1],
\end{align*}
where $v(t, x): [0, 1] \times \sR^d \to \sR^d$ and $\tilde{v}(t, x): [0, 1] \times \sR^d \to \sR^d$ are the velocity fields.
\begin{itemize}
\item[(i)] Suppose that $v$ is of class $C^1$ in $x$. Then the Alekseev-Gr{\"o}bner formula for the difference $X_t(x_0) - Y_t(x_0)$ is given by
\begin{equation}
    \label{eq:perb-vf-x}
    X_t(x_0) - Y_t(x_0) = \int_0^t (\nabla_x X_{s,t}) (Y_s(x_0))^{\top} \left( v(s, Y_s(x_0)) - \tilde{v}(s, Y_s(x_0)) \right) \diff s
\end{equation}
where $\nabla_x X_{s,t}(x)$ satisfies the variational equation
\begin{equation}
    \label{eq:vari-eq-x}
    \partial_t (\nabla_x X_{s,t}(x)) = (\nabla_x v) (t, X_{s,t}(x)) \nabla_x X_{s,t}(x), \quad \nabla_x X_{s,s}(x) = \rmI_d.
\end{equation}
\item[(ii)] Suppose that $\tilde{v}$ is of class $C^1$ in $x$. Then the Alekseev-Gr{\"o}bner formula for the difference $Y_t(x_0) - X_t(x_0)$ is given by
\begin{equation}
    \label{eq:perb-vf-y}
    Y_t(x_0) - X_t(x_0) = \int_0^t (\nabla_x Y_{s,t}) (X_s(x_0))^{\top} \left( \tilde{v}(s, X_s(x_0)) - v(s, X_s(x_0)) \right) \diff s
\end{equation}
where $\nabla_x Y_{s,t}(x)$ satisfies the variational equation
\begin{equation}
    \label{eq:vari-eq-y}
    \partial_t (\nabla_x Y_{s,t}(x)) = (\nabla_x \tilde{v}) (t, Y_{s,t}(x)) \nabla_x Y_{s,t}(x), \quad \nabla_x Y_{s,s}(x) = \rmI_d.
\end{equation}
\end{itemize}
\end{lemma}

Exploiting the Alekseev-Gr{\"o}bner formulas in Lemma \ref{lm:ag-formula} and uniform Lipschitz properties of the velocity field, we deduce two error bounds in terms of the quadratic Wasserstein ($W_2$) distance to show the stability of the ODE flow when the velocity field is not accurate.

\begin{proposition}[Stability in the velocity field] \label{prop:stab-vf}
Suppose Assumptions \ref{assump:well-defined} and \ref{assump:geom-prop} hold. Let $\tilde{q}_t$ denote the density function of ${Y_t}_{\#} \mu$.
\begin{itemize}
\item[(i)] Suppose that
\begin{equation}
    \label{eq:infty-bd-vf}
    \int_0^1 \int_{\sR^d} \Vert v(t, x) - \tilde{v}(t, x) \Vert^2 \tilde{q}_t(x) \diff x \diff t \le \varepsilon.
\end{equation}
Then
\begin{equation}
    \label{eq:w2-1-stable}
    W_2^2({Y_1}_{\#} \mu, \nu) \le \varepsilon \int_0^1 \exp \left( 2 \int_s^1 \theta_u \diff u \right) \diff s.
\end{equation}
\item[(ii)] Suppose that
\begin{equation*}
    \sup_{(t, x) \in [0, 1] \times \sR^d} \Vert \nabla_x \tilde{v}(t, x) \Vert_{2,2} \le C_3.
\end{equation*}
Then
\begin{equation}
    \label{eq:w2-2-stable}
    W_2^2({Y_1}_{\#} \mu, \nu) \le \frac{\exp(2 C_3) - 1}{2 C_3} \int_0^1 \int_{\sR^d} \Vert v(t, x) - \tilde{v}(t, x) \Vert^2 p_t(x) \diff x \diff t.
\end{equation}
\end{itemize}
\end{proposition}

Proposition \ref{prop:stab-vf} provides a stability analysis against the estimation error of the velocity field using the $W_2$ distance.
    The estimation error originates from the flow matching or score matching procedures and
    the approximation error rising from using
     deep neural networks in
     estimating the velocity field or the score function. These two $W_2$ bounds imply that the distribution estimation error is controlled by the $L_2$ estimation error of flow matching and score matching.
    Indeed, this point justifies the soundness of the approximation method through flow matching and score matching.
    The first $W_2$ bound \paref{eq:w2-1-stable} relies on the $L_2$ control \paref{eq:infty-bd-vf} of the perturbation error of the velocity field.
    The second $W_2$ bound \paref{eq:w2-2-stable} is slightly better than that provided in \cite[Proposition 3]{albergo2023building} but still has exponential dependence on the Lipschitz constant of $\tilde{v}(t, x)$.

\begin{figure}[H]
\centering
\includegraphics[width=0.65\textwidth]{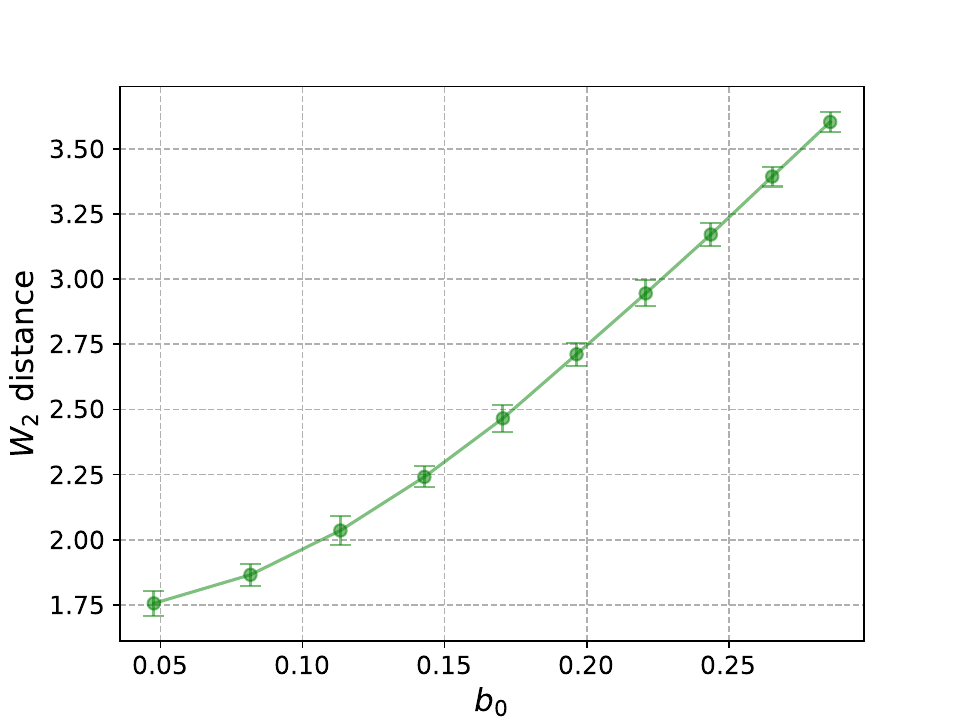}
\caption{An approximately linear relation between $b_0$ and the Wasserstein-2 distance.}
\label{fig:mog-perturb-source}
\end{figure}

To demonstrate the bounds presented in Propositions \ref{prop:stab-iv} and \ref{prop:stab-vf}, we conducted further experiments with a mixture of eight two-dimensional Gaussian distributions. These propositions provide bounds for the stability of the flow when subjected to perturbations in either the source distribution or the velocity field.
Let the target distribution be the following two-dimensional Gaussian mixture
\begin{align*}
p(x)=\sum_{j=1}^{8}\phi(x; \mu_j, \Sigma_j),
\end{align*}
where $\phi(x; \mu_j, \Sigma_j)$ is the probability density function for the Gaussian distribution with mean $\mu_j = 12 (\sin(2(j-1)\pi/8),\cos(2(j-1)\pi/8))^{\top}$ and covariance matrix $\Sigma_j= 0.03^2 \mathbf{I}_2$ for $j=1, \cdots, 8$.  For Gaussian mixtures, the velocity field has an explicit formula, which facilitates the perturbation analysis.

To illustrate the bound in Proposition \ref{prop:stab-iv}, we consider a perturbation of the source distribution for the following model:
\begin{align*}
    \mathsf{X}_t = a_t \mathsf{Z} + b_t \mathsf{X} \text{\quad with  \quad}
    a_t = 1 - \frac{t + \zeta}{1 + \zeta}, \quad b_t = \frac{t + \zeta}{1 + \zeta},
\end{align*}
where $\zeta \in [0, 0.3]$ is a value controlling the perturbation level.
It is easy to see $a_0 = {1}/{(1 + \zeta)}, b_0 = {\zeta}/{(1 + \zeta)}$.
Thus, the source distribution $\mathrm{Law}(a_0 \mathsf{Z} + b_0 \mathsf{X})$ is a mixture of Gaussian distributions.
Practically, we can use a Gaussian distribution $\gamma_{2, a_0^2}$ to replace this source distribution.
In Proposition \ref{prop:stab-iv}, we bound the error between the distributions of generated samples due to the replacement, that is,
\begin{align*}
    W_2({X_1}_{\#} \gamma_{d, a_0^2}, \nu) \le C b_0,
\end{align*}
where $C$ is a constant.
We illustrate this theoretical bound using the mixture of Gaussian distributions and the Gaussian interpolation flow given above.
We consider a mesh for the variable $\zeta$ and plot the curve for $b_0$ and $W_2({X_1}_{\#} \gamma_{d, a_0^2}, \nu)$ in Figure \ref{fig:mog-perturb-source}.
Through Figure \ref{fig:mog-perturb-source}, an approximate linear relation between $b_0$ and $W_2({X_1}_{\#} \gamma_{d, a_0^2}, \nu)$ is observed, which supports the results of Proposition \ref{prop:stab-iv}.

\begin{figure}[!t]
\centering
\includegraphics[width=0.65\textwidth]{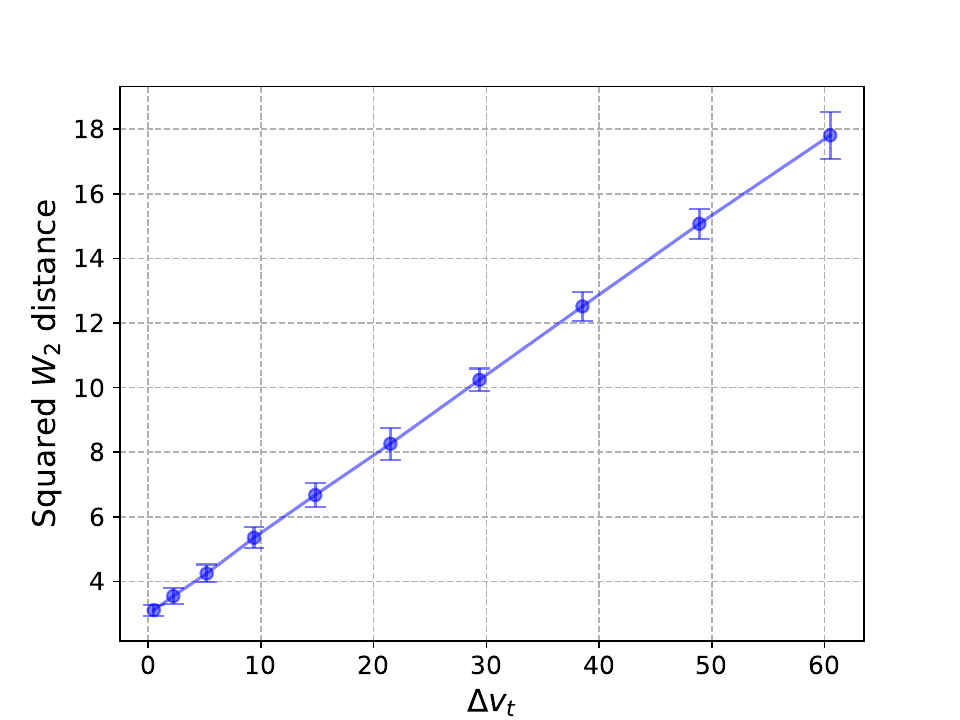}
\caption{A linear relation between $\Delta v_t$ and the squared Wasserstein-2 distance.}
\label{fig:mog-perturb-vf}
\end{figure}

We now consider perturbing the velocity field $v_t$ by adding random noise. Let $\epsilon \in [0.5, 5.5]$. The random noise is generated using a Bernoulli random variable supported on $\{ -\epsilon, \epsilon \}$. Let $\tilde{v}_t$ denote the perturbed velocity field. Then we can compute
\begin{align*}
    \Delta v_t := \Vert v_t - \tilde{v}_t \Vert^2 = 2 \epsilon^2.
\end{align*}
We use the velocity field $v_t$ and the perturbed velocity field $\title{v}_t$ to generate samples and compute the squared Wasserstein-2 distance between the sample distributions.
According to Proposition \ref{prop:stab-vf}, the squared Wasserstein-2 distance should be linearly upper bounded as $\mathcal{O}(\Delta v_t)$, that is,
\begin{align*}
    W_2^2({Y_1}_{\#} \mu, \nu) \le \tilde{C} \int_0^1 \int_{\mathbb{R}^2} \epsilon^2 p_t(x) \mathrm{d} x \mathrm{d} t = \tilde{C} \epsilon^2,
\end{align*}
where $\tilde{C}$ is a constant. This theoretical insight is illustrated in Figure \ref{fig:mog-perturb-vf}, where a linear relationship between these two variables is observed.

\section{Related work}

GIFs and the induced transport maps are
related to CNFs and score-based diffusion models.
Mathematically, they interrelate with the literature on Lipschitz mass transport and Wasserstein gradient flows.
A central question in developing the ODE flow or transport map method for generative modeling is how to construct
an ODE flow or transport map that are
sufficiently smooth and enable efficient computation.
Various approaches have been proposed to answer the question.

CNFs construct invertible mappings between an isotropic Gaussian distribution and a complex target distribution \citep{chen2018neural, grathwohl2018scalable}. They fall within the broader framework of neural ODEs \citep{chen2018neural, ruiz2023neural}. A major challenge for CNFs is designing a time-dependent ODE flow whose marginal distribution converges to the target distribution while allowing for efficient estimation of its velocity field. Previous work has explored several principles to construct such flows, including optimal transport, Wasserstein gradient flows, and diffusion processes. Additionally, Gaussian denoising has emerged as an effective principle for constructing simulation-free CNFs in generative modeling.

\citet{liu2023flow} propose the rectified flow, which is based on a linear interpolation between a standard Gaussian distribution and the target distribution, mimicking the Gaussian denoising procedure. \citet{albergo2023building} study a similar formulation called stochastic interpolation, defining a trigonometric interpolant between a standard Gaussian distribution and the target distribution. \citet{albergo2023stochastic} extend this idea by proposing a stochastic bridge interpolant between two arbitrary distributions. Under a few regularity assumptions, the velocity field of the ODE flow modeling the stochastic bridge interpolant is proven to be continuous in the time variable and smooth in the space variable. 

\citet{lipman2023flow} introduce a nonlinear least squares method called flow matching to directly estimate the velocity field of probability flow ODEs. All of these models are encompassed within the framework of simulation-free CNFs, which have been the focus of numerous ongoing research efforts \citep{neklyudov2023action, tong2023conditional, chen2023riemannian, albergo2023stochastic, shaul2023kinetic, pooladian2023multisample, albergo2023multimarginal, albergo2023dependent}. Furthermore, \citet{marzouk2023distribution} provide the first statistical convergence rate for the simulation-based method by placing neural ODEs within the nonparametric estimation framework.

Score-based diffusion models integrate the time reversal of stochastic differential equations (SDEs) with the score matching technique \citep{sohl2015deep, song2019generative, ho2020denoising, song2020improved, song2021scorebased, song2021denoising, debortoli2021diffusion}. These models are capable of modeling highly complex probability distributions and have achieved state-of-the-art performance in image synthesis tasks \citep{dhariwal2021diffusion, rombach2022high}. The probability flow ODEs of diffusion models can be considered as CNFs, whose velocity field incorporates the nonlinear score function \citep{song2021scorebased, karras2022elucidating, lu2022dpm, lu2022maximum, zheng2023improved}.
In addition to the score matching method, \citet{lu2022maximum} and \citet{zheng2023improved} explore maximum likelihood estimation for probability flow ODEs. However, the regularity of these probability flow ODEs has not been studied and their well-posedness properties remain to be established.

A key concept in defining measure transport is Lipschitz mass transport, where the transport maps are required to be Lipschitz continuous. This ensures the smoothness and stability of the measure transport. There is a substantial body of research on the Lipschitz properties of transport maps. The celebrated Caffarelli's contraction theorem \cite[Theorem 2]{caffarelli2000monotonicity} establishes the Lipschitz continuity of optimal transport maps that push the standard Gaussian measure onto a log-concave measure. \citet{colombo2017lipschitz} study a Lipschitz transport map between perturbations of log-concave measures using optimal transport theory.

\citet{mikulincer2021brownian} demonstrate that the Brownian transport map, defined by the F{"o}llmer process, is Lipschitz continuous when it pushes forward the Wiener measure on the Wiener space to the target measure on the Euclidean space. Additionally, \citet{neeman2022lipschitz} and \citet{mikulincer2023lipschitz} prove that the transport map along the reverse heat flow of certain target measures is Lipschitz continuous.

Beyond studying Lipschitz transport maps, significant effort has been devoted to applying optimal transport theory in generative modeling. \citet{zhang2018monge} propose the Monge-Ampe{`r}e flow for generative modeling by solving the linearized Monge-Ampe{\`r}e equation. Optimal transport theory has been utilized as a general principle to regularize the training of continuous normalizing flows or generators for generative modeling \citep{finlay2020train, yang2020potential, onken2021ot, makkuva2020optimal}. \citet{liang2021well} leverage the regularity theory of optimal transport to formalize the generator-discriminator-pair regularization of GANs under a minimax rate framework.

In our work, we study the Lipschitz transport maps defined by GIFs, which differ from the optimal transport map. GIFs naturally fit within the framework of continuous normalizing flows, and their flow mappings are examined from the perspective of Lipschitz mass transport.

Wasserstein gradient flows offer another principled approach to constructing ODE flows for generative modeling. A Wasserstein gradient flow is derived from the gradient descent minimization of a certain energy functional over probability measures endowed with the quadratic Wasserstein metric \citep{ambrosio2008gradient}. The Eulerian formulation of Wasserstein gradient flows produces the continuity equations that govern the evolution of marginal distributions.
After transferred into a Lagrangian formulation, Wasserstein gradient flows define ODE flows that have been widely explored for generative modeling \citep{johnson2018composite, gao2019deep, liutkus2019sliced, johnson2019framework, arbel2019maximum, mroueh2019sobolev, ansari2021refining, mroueh2021convergence, fan2021variational, gao2022deep, duncan2023geometry, xu2022invertible}.
Wasserstein gradient flows are shown to be connected with the forward process of diffusion models.
The variance preserving SDE of diffusion models is equivalent to the Langevin dynamics towards the standard Gaussian distribution that can be interpreted as a Wasserstein gradient flow of the Kullback–Leibler divergence for a standard Gaussian distribution \citep{song2021scorebased}.
In the meantime, the probability flow ODE of the variance preserving SDE conforms to the Eulerian formulation of this Wasserstein gradient flow.
However, when assigning a general distribution instead of the standard Gaussian distribution, it remains unclear whether the ODE formulation of Wasserstein gradient flows possesses well-posedness.

The main contribution of our work lies in establishing the theoretical properties of GIFs and their associated flow maps in a unified way. Our theoretical results encompass the Lipschitz continuity of both the flow's velocity field and the flow map, addressing the existence, uniqueness, and stability of the flow. We also demonstrate that both the flow map and its inverse possess Lipschitz properties.

Our proposed framework for Gaussian interpolation flow builds upon previous research on probability flow methods in diffusion models \citep{song2021scorebased, song2021denoising} and stochastic interpolation methods for generative modeling \citep{liu2023flow, albergo2023building, lipman2023flow}. Rather than adopting a methodological perspective, we focus on elucidating the theoretical aspects of these flows from a unified standpoint, thereby enhancing the understanding of various methodological approaches. Our theoretical results are derived from geometric considerations of the target distribution and from analytic calculations that exploit the Gaussian denoising property.

\section{Conclusions and discussion}

Gaussian denoising as a framework for constructing continuous normalizing flows holds great promise in generative modeling. Through a unified framework and rigorous analysis, we have established the well-posedness of these flows, shedding light on their capabilities and limitations. We have examined the Lipschitz regularity of the corresponding flow maps for several rich classes of probability measures. When applied to generative modeling based on Gaussian denoising, we have shown that GIFs possess auto-encoding and cycle consistency properties at the population level. Additionally, we have established stability error bounds for the errors accumulated during the process of learning GIFs.

The regularity properties of the velocity field established in this paper provide a solid theoretical basis for end-to-end error analyses of learning GIFs using deep neural networks with empirical data. Another potential application is to perform rigorous analyses of consistency models, a nascent family of ODE-based deep generative models designed for one-step generation \citep{song2023consistency, kim2023consistency, song2023improved}. We intend to investigate these intriguing problems in our subsequent work. We expect that our analytical results will facilitate further studies and advancements in applying simulation-free CNFs, including GIFs, to a diverse range of generative modeling tasks.




\newpage

\appendix
\noindent
\textbf{\large Appendix}

\medskip\noindent
{\color{black}
In the appendices, we prove the results stated in the paper and provide necessary technical details and discussions.

\section{Proofs of Theorem \ref{thm:flow-gif} and Lemma \ref{lm:cond-cov}}
Dynamical properties of Gaussian interpolation flow $(\mathsf{X}_t)_{t \in [0, 1]}$ form the cornerstone of the measure interpolation method. Following \citet{albergo2023building, albergo2023stochastic}, we leverage an argument of characteristic functions to quantify the dynamics of its marginal flow, and in result, to prove Theorem \ref{thm:flow-gif}.
\begin{proof} [Proof of Theorem \ref{thm:flow-gif}]
Let $\omega \in \sR^d$.
For the Gaussian stochastic interpolation $(\mathsf{X}_t)_{t \in [0, 1]}$, we define the characteristic function of $\mathsf{X}_t$ by
\begin{equation*}
    \Psi(t, \omega)
    := \E [\exp(i \langle \omega, \mathsf{X}_t \rangle)]
    = \E [\exp(i \langle \omega, a_t \mathsf{Z} + b_t \mathsf{X}_1 \rangle)]
    = \E [\exp(i a_t \langle \omega, \mathsf{Z} \rangle)] \E [\exp(i b_t \langle \omega, \mathsf{X}_1 \rangle)],
\end{equation*}
where the last equality is due to the independence of between $\mathsf{Z} \sim \gamma_d$ and $\mathsf{X}_1 \sim \nu$.
Taking the time derivative of $\Psi(t, \omega)$ for $t \in (0, 1)$, we derive that
\begin{equation*}
    \partial_t \Psi(t, \omega) = i \langle \omega, \psi(t, \omega))
\end{equation*}
where
\begin{equation*}
    \psi(t, \omega) := \E [\exp(i \langle \omega, \mathsf{X}_t \rangle) (\dot{a}_t \mathsf{Z} + \dot{b}_t \mathsf{X}_1)].
\end{equation*}
We first define
\begin{align} \label{eq:def-vf}
    v(t, \mathsf{X}_t) := \E[ \dot{a}_t \mathsf{Z} + \dot{b}_t \mathsf{X}_1 | \mathsf{X}_t].
\end{align}
Using the double expectation formula, we deduce that
\begin{align*}
    \psi(t, \omega)
    = \E[ \exp(i \langle \omega, \mathsf{X}_t \rangle) \E[\dot{a}_t \mathsf{Z} + \dot{b}_t \mathsf{X}_1 | \mathsf{X}_t] ]
    = \E[ \exp(i \langle \omega, \mathsf{X}_t \rangle) v(t, \mathsf{X}_t) ].
\end{align*}
Applying the inverse Fourier transform to $\psi(t, \omega)$, it holds that
\begin{equation*}
    j(t, x)
    := (2 \pi)^{-d} \int_{\mathbb{R}^d}\exp(-i \langle \omega, x \rangle) \psi(t, \omega) \mathrm{d} \omega
    = p_t(x) v(t, x),
\end{equation*}
where $v(t, x) := \E[ \dot{a}_t \mathsf{Z} + \dot{b}_t \mathsf{X}_1 | \mathsf{X}_t = x]$.
Then it further yields that
\begin{equation*}
    \partial_t p_t + \nabla_x \cdot j(t, x) = 0,
\end{equation*}
that is,
\begin{equation*}
    \partial_t p_t + \nabla_x \cdot (p_t v(t, x)) = 0.
\end{equation*}
Next, we study the property of $v(t, x)$ at $t = 0$ and $t = 1$.
Notice that
\begin{align} \label{eq:def-cond-exp}
    x = a_t \E[\mathsf{Z} | \mathsf{X}_t = x] + b_t \E[\mathsf{X}_1 | \mathsf{X}_t = x].
\end{align}
Combining Eq. \paref{eq:def-vf} and \paref{eq:def-cond-exp}, it implies that
\begin{equation} \label{eq:vf-expect-single2}
    v(t, x) = \tfrac{\dot{a}_t}{a_t} x + \left( \dot{b}_t - \tfrac{\dot{a}_t}{a_t}b_t \right) \E[\mathsf{X}_1 | \mathsf{X}_t = x], \quad t \in (0, 1).
\end{equation}
According to Tweedie's formula in Lemma \ref{lm:tw-formula}, it holds that
\begin{equation}
    \label{eq:tw-formula2}
    s(t, x) = \tfrac{b_t}{a_t^2} \E \left[ \mathsf{X}_1 | \mathsf{X}_t = x \right] - \tfrac{1}{a_t^2} x, \quad t \in (0, 1),
\end{equation}
where $s(t, x)$ is the score function of the marginal distribution of $\mathsf{X}_t \sim p_t$.

Combining Eq. \paref{eq:vf-expect-single2}, \paref{eq:tw-formula2}, it holds that the velocity field is a gradient field and its nonlinear term is the score function $s(t, x)$, namely, for any $t \in (0, 1)$,
\begin{align}
    \label{eq:gene-vf-interpolation}
    v(t, x) = \tfrac{\dot{b}_t}{b_t} x + \left( \tfrac{\dot{b}_t}{b_t} a_t^2 - \dot{a}_t a_t \right) s(t, x).
\end{align}
By the regularity properties that $a_t, b_t \in C^2([0,1)), a_t^2 \in C^1([0, 1]), b_t \in C^1([0, 1])$, we have that $\dot{a}_0, \dot{b}_0, \dot{a}_1 a_1$, and $\dot{b}_1$ are well-defined.
Then by Eq. \paref{eq:vf-expect-single2}, we define that
\begin{align*}
    v(0, x) := \lim_{t \downarrow 0} v(t, x)
    = \tfrac{\dot{a}_0}{a_0} x + \left( \dot{b}_0 - \tfrac{\dot{a}_0}{a_0}b_0 \right) \E[\mathsf{X}_1 | \mathsf{X}_0 = x]
\end{align*}
Using Eq. \paref{eq:gene-vf-interpolation} yields that
\begin{align} \label{eq:vf-time1}
    v(1, x) := \lim_{t \uparrow 1} v(t, x)
    = \tfrac{\dot{b}_1}{b_1} x - \dot{a}_1 a_1 s(1, x).
\end{align}
This completes the proof.
\end{proof}

Lemma \ref{lm:cond-cov} presents several standard properties of Gaussian channels in information theory \citep{wibisono2018convexity-heat, wibisono2018convexity-ou, dytso2023conditional} that will facilitate our proof.

\begin{proof} [Proof of Lemma \ref{lm:cond-cov}]
{\color{black}
By Bayes' rule, $\mathrm{Law} (\mathsf{Y} | \mathsf{X}_t = x) = p(y | t, x)$ can be represented as
\begin{align*}
    p(y | t, x)
    & = \varphi_{b_t y, a_t^2} (x) p_1(y) / p_t(x) \\
    & = (2 \pi)^{-d/2} a_t^{-d} \exp\left( -\frac{\Vert x - b_t y \Vert^2}{2 a_t^2} \right) p_1(y) / p_t(x) \\
    & = (2 \pi)^{-d/2} a_t^{-d} \exp\left( -\frac{\Vert x \Vert^2}{2 a_t^2} + \frac{b_t \langle x, y \rangle}{a_t^2} - \frac{b_t^2 \Vert y \Vert^2}{2 a_t^2} \right) p_1(y) / p_t(x) \\
    & = \left\{ \exp\left( \frac{b_t \langle x, y \rangle}{a_t^2} - \frac{b_t^2 \Vert y \Vert^2}{2 a_t^2} \right) p_1(y) \right\}
    / \left\{ (2 \pi)^{d/2} a_t^d \exp\left(\frac{\Vert x \Vert^2}{2 a_t^2}\right) p_t(x) \right\}.
\end{align*}
Let $\theta = \frac{b_t x}{a_t^2}, h(y) = p_1(y) \exp(-\frac{b_t^2 \Vert y \Vert^2}{2 a_t^2})$, and the logarithmic partition function
\begin{equation*}
    A(\theta) = \log \int_{\sR^d} h(y) \exp(\langle y, \theta \rangle) \diff y,
\end{equation*}
then by the definition of exponential family distributions, we conclude that
\begin{align*}
    p(y | t, x) = h(y) \exp(\langle y, \theta \rangle - A(\theta))
\end{align*}
is an exponential family distribution of $y$.}
By simple calculation, it follows that
\begin{equation*}
    \nabla^2_x \log p(y | t, x) =  - \frac{b_t^2}{a_t^4} \nabla^2_{\theta} A(\theta).
\end{equation*}
For an exponential family distribution, a basic equality shows that
\begin{equation*}
    \nabla^2_{\theta} A(\theta) = \Cov (\mathsf{Y} | \mathsf{X}_t = x),
\end{equation*}
which further yields that $\nabla^2_x \log p(y | t, x) = - \frac{b_t^2}{a_t^4} \Cov (\mathsf{Y} | \mathsf{X}_t = x)$.
\end{proof}

\section{Auxiliary lemmas for Lipschitz flow maps}
The following lemma, due to G. Peano \cite[Theorem 3.1]{hartman2002dependence}, describes several meaningful differential equations associated with well-posed flows and supports the derivation of Lipschitz continuity of their flow maps.

\begin{lemma} \label{lm:diff-eq-flow} \cite[Lemma 3.4]{ambrosio2023classical}
Suppose that a flow $(X_t)_{t \in [0, 1]}$ is well-posed and its velocity field $v(t, x): [0, 1] \times \sR^d \to \sR^d$ is of class $C^1$. Then the flow map $X_{s,t}: \sR^d \to \sR^d$ is of class $C^1$ for any $0 \le s \le t \le 1$. Fix $(s, x) \in [0, 1] \times \sR^d$ and set the following functions defined with $t \in [s, 1]$
\begin{align*}
    y(t) &:= \nabla_x X_{s,t}(x),        && J(t) := (\nabla_x v) (t, X_{s,t}(x)), \\
    w(t) &:= \det (\nabla_x X_{s,t}(x)), && b(t) := (\nabla_x \cdot v) (t, X_{s,t}(x)) = \Tr(J(t)).
\end{align*}
Then $y(t)$ and $w(t)$ are the unique $C^1$ solutions of the following IVPs
\begin{align}
    \label{eq:grad-eq}
    \dot{y}(t) &= J(t)y(t), \quad y(s) = \rmI_d, \\
    \label{eq:det-eq}
    \dot{w}(t) &= b(t) w(t), \quad w(s) = 1.
\end{align}
\end{lemma}

{\color{black} We present an upper bound of the Lipschitz constant of its flow map $X_{s, t}(x)$ in Lemma \ref{lm:flow-map-Lip-bd}.}
The upper bound has been deduced in \citet{mikulincer2023lipschitz, ambrosio2023classical, dai2023lipschitz}. For completeness, we derive it as a direct implication of Eq. \paref{eq:grad-eq} in
{\color{black} Lemma \ref{lm:diff-eq-flow} }
and an upper bound of the Jacobian matrix of the velocity field.

\begin{proof}[Proof of Lemma \ref{lm:flow-map-Lip-bd}]
Let $y(u) = \nabla_x X_{s,u}(x)$, $J(u) = (\nabla_x v) (u, X_{s,u}(x))$.
Owing to Lemma \ref{lm:diff-eq-flow}, $y(u)$ is of class $C^1$, and the function $u \mapsto \Vert y(u) \Vert_{2,2}$ is absolutely continuous over $[s, t]$. By Lemma \ref{lm:diff-eq-flow}, it follows that
\begin{equation*}
    \partial_u \Vert y(u) \Vert_{2,2}^2 = 2 \langle y(u), \dot{y}(u) \rangle
    = 2 \langle y(u), J(u) y(u) \rangle \le 2 \theta_u \Vert y(u) \Vert_{2,2}^2.
\end{equation*}
Applying Gr{\"o}nwall's inequality yields that $\Vert y(t) \Vert_{2,2} \le \exp( \int_s^t \theta_u \diff u)$ which concludes the proof.
\end{proof}

Another result is concerning the theorem of instantaneous change of variables that is widely deployed in studying neural ODEs \cite[Theorem 1]{chen2018neural}. We also exploit the instantaneous change of variables to prove Proposition \ref{prop:stab-iv}.
To make the proof self-contained, we show that the instantaneous change of variables directly follows Eq. \paref{eq:det-eq} in Lemma \ref{lm:diff-eq-flow}.
Compared with the original proof in \cite[Theorem 1]{chen2018neural}, we illustrate that the well-posedness of a flow is sufficient to ensure the instantaneous change of variables property, without a boundedness condition on the flow.

\begin{corollary} [Instantaneous change of variables] \label{cor:change-var}
Suppose that a flow $(X_t)_{t \in [0, 1]}$ is well-posed with a velocity field $v(t, x): [0, 1] \times \sR^d \to \sR^d$ of class $C^1$ in $x$. Let $X_0(x) \sim \pi_0(X_0(x))$ be a distribution of the initial value.
Then the law of $X_t(x)$ satisfies the following differential equation
\begin{equation*}
    \partial_t \log \pi_t(X_t(x)) = - \Tr((\nabla_x v) (t, X_t(x))).
\end{equation*}
\end{corollary}

\begin{proof}
Let $\delta(t) := \det (\nabla_x X_t (x))$. Thanks to Eq. \paref{eq:det-eq} in Lemma \ref{lm:diff-eq-flow}, it holds that
\begin{equation*}
    \dot{\delta}(t) = \Tr ((\nabla_x v) (t, X_t(x))) \delta(t), \quad \delta(0) = 1,
\end{equation*}
which implies $\delta(t) > 0$ for $t \in [0, 1]$.
Notice that $\log \pi_t(X_t(x)) = \log \pi_0 (X_0(x)) - \log |\delta(t)|$ by change of variables.
Then it follows that $\partial_t \log \pi_t(X_t(x)) = - \Tr ((\nabla_x v) (t, X_t(x)))$.
\end{proof}

\section{Proofs of spatial Lipschitz estimates for the velocity field}

The main results in Section \ref{sec:spatial-lip} are proved in this appendix. We first present some ancillary lemmas before proceeding to give the proofs.

\begin{lemma} [\citealp{fathi2023transportation}] \label{lm:log-lip}
Suppose that $f: \sR^d \to \sR_+$ is $L$-log-Lipschitz for some $L \ge 0$.
Let $\mathcal{P}_t$ be the Ornstein–Uhlenbeck semigroup defined by $\mathcal{P}_t h(x) := \E_{\mathsf{Z} \sim \gamma_d} [h(e^{-t} x + \sqrt{1- e^{-2t}} \mathsf{Z})]$ for any $h \in C(\sR^d)$ and $t \ge 0$.
Then it holds that
\begin{equation*}
    \left\{ -5L e^{-t} (L + t^{-\frac12}) - L^2 e^{-2t} \right\} \rmI_d
    \preceq \nabla^2_x \log \mathcal{P}_t f(x)
    \preceq \left\{ 5L e^{-t} (L + t^{-\frac12}) \right\} \rmI_d.
\end{equation*}
\end{lemma}

\begin{proof}
    This is a restatement of known results. See Proposition 2, Proposition 6, Theorem 6, and their proofs in \citet{fathi2023transportation}.
\end{proof}

\begin{corollary} \label{cor:bd-op}
Suppose that $f: \sR^d \to \sR_+$ is $L$-log-Lipschitz for some $L \ge 0$.
Let $\mathcal{Q}_t$ be an operator defined by
\begin{equation}
    \label{eq:op-ab}
    \mathcal{Q}_t h(x) := \E_{\mathsf{Z} \sim \gamma_d} [h(\beta_t x + \alpha_t \mathsf{Z})]
\end{equation}
for any $h \in C(\sR^d)$ and $t \in [0, 1]$ where $0 \le \alpha_t \le 1, \beta_t \ge 0$ for any $t \in [0, 1]$.
Then it holds that
\begin{align*}
    \left(-A_t - L^2 \beta_t^2 \right) \rmI_d
    \preceq \nabla^2_x \log \mathcal{Q}_t f(x)
    \preceq A_t \rmI_d,
\end{align*}
where $A_t := 5L \beta_t^2 (1 - \alpha_t^2)^{-\frac12} (L + (-\frac12 \log (1 - \alpha_t^2))^{-\frac12})$.
\end{corollary}

\begin{proof}
It is easy to notice that $\mathcal{Q}_t f(x) = \mathcal{P}_s f (\beta_t e^s x)$ where $s = -\frac12 \log (1 - \alpha_t^2)$.
Then it follows that $\nabla^2_x \log \mathcal{Q}_t f(x) = (\beta_t e^s)^2 (\nabla^2_x \log \mathcal{P}_s f) (\beta_t e^s x)$ which yields
\begin{align*}
    \left(-A_t - L^2 \beta_t^2 \right) \rmI_d
    \preceq \nabla^2_x \log \mathcal{Q}_t f(x)
    \preceq A_t \rmI_d,
\end{align*}
where $A_t := 5L \beta_t^2 (1 - \alpha_t^2)^{-\frac12} (L + (-\frac12 \log (1 - \alpha_t^2))^{-\frac12})$.
\end{proof}

\begin{lemma} \label{lm:vf-op}
The Jacobian matrix of the velocity field \paref{eq:vf-expect} has an alternative expression over time $t \in (0, 1)$, that is,
\begin{equation*}
    \nabla_x v(t, x) = \left( \tfrac{\dot{b}_t}{b_t} a_t^2 - \dot{a}_t a_t \right) \left( \nabla^2_x \log \widetilde{\mathcal{Q}}_t f(x) - \tfrac{1}{a_t^2 + b_t^2} \rmI_d \right) + \tfrac{\dot{b}_t}{b_t} \rmI_d,
\end{equation*}
where $f(x) := \frac{\diff \nu}{\diff \gamma_d}(x)$ and $\widetilde{\mathcal{Q}}_t f(x) := \E_{\mathsf{Z} \sim \gamma_d} [ f ( \frac{b_t}{a_t^2 + b_t^2} x + \frac{a_t}{\sqrt{a_t^2 + b_t^2}} \mathsf{Z}) ]$.
\end{lemma}

\begin{proof}
By direct calculations, it holds that
\begin{align*}
    p_t(x)
    &= a_t^{-d} \int_{\sR^d} p_1(y) \varphi \left( \frac{x-b_t y}{a_t} \right) \diff y
     = a_t^{-d} \int_{\sR^d} f(y) \varphi(y) \varphi \left( \frac{x-b_t y}{a_t} \right) \diff y \\
    &= a_t^{-d} \varphi \left( (a_t^2 + b_t^2)^{-\frac12} x \right) \int_{\sR^d} f(y) \varphi \left( \left( \frac{a_t}{\sqrt{a_t^2 + b_t^2}} \right)^{-1} \left( y - \frac{b_t}{a_t^2 + b_t^2} x \right) \right) \diff y \\
    &= a_t^{-d} \varphi \left( (a_t^2 + b_t^2)^{-\frac12} x \right) \left(\frac{a_t}{\sqrt{a_t^2 + b_t^2}}\right)^d \int_{\sR^d} f \left( \frac{b_t}{a_t^2 + b_t^2} x + \frac{a_t}{\sqrt{a_t^2 + b_t^2}} z \right) \diff \gamma_d(z) \\
    &= (a_t^2 + b_t^2)^{-d/2} \varphi \left( (a_t^2 + b_t^2)^{-\frac12} x \right) \widetilde{\mathcal{Q}}_t f(x).
\end{align*}
Taking the logarithm and then the second-order derivative of the equation above, it yields
\begin{equation*}
     \nabla_x s(t, x) = \nabla^2_x \log \widetilde{\mathcal{Q}}_t f(x) - \tfrac{1}{a_t^2 + b_t^2} \rmI_d.
\end{equation*}
Recalling that $\nabla_x v(t, x) = \left( \frac{\dot{b}_t}{b_t} a_t^2 - \dot{a}_t a_t \right) \nabla_x s(t, x) + \frac{\dot{b}_t}{b_t} \rmI_d$, it further yields that
\begin{equation*}
    \nabla_x v(t, x) = \left( \tfrac{\dot{b}_t}{b_t} a_t^2 - \dot{a}_t a_t \right) \nabla^2_x \log \widetilde{\mathcal{Q}}_t f(x) + \tfrac{\dot{a}_t a_t + \dot{b}_t b_t}{a_t^2 + b_t^2} \rmI_d,
\end{equation*}
which completes the proof.
\end{proof}

\begin{corollary} \label{cor:bd-vf-log-Lip}
Suppose that $f(x) := \frac{\diff \nu}{\diff \gamma_d}(x)$ is $L$-log-Lipschitz for some $L \ge 0$. Then for $t \in (0, 1)$, it holds that
\begin{align*}
    & \left\{ \left( \tfrac{\dot{b}_t}{b_t} a_t^2 - \dot{a}_t a_t \right) \left( -B_t - L^2 \left( \tfrac{b_t}{a_t^2 + b_t^2} \right)^2 \right) + \tfrac{\dot{a}_t a_t + \dot{b}_t b_t}{a_t^2 + b_t^2} \right\} \rmI_d \\
    & \preceq \nabla_x v(t, x)
    \preceq \left\{ \left( \tfrac{\dot{b}_t}{b_t} a_t^2 - \dot{a}_t a_t \right) B_t + \tfrac{\dot{a}_t a_t + \dot{b}_t b_t}{a_t^2 + b_t^2} \right\} \rmI_d,
\end{align*}
where $B_t := 5L b_t (a_t^2 + b_t^2)^{-\frac32} (L + (\log (\sqrt{a_t^2 + b_t^2}/b_t))^{-\frac12})$.
\end{corollary}

\begin{proof}
    Let $\alpha_t = \frac{a_t}{\sqrt{a_t^2 + b_t^2}}$ and $\beta_t = \frac{b_t}{a_t^2 + b_t^2}$. Then these bounds hold according to Corollary \ref{cor:bd-op} and Lemma \ref{lm:vf-op}.
\end{proof}

Then we are prepared to prove Proposition \ref{prop:vf-bd}. The proof is mainly based on the techniques for bounding conditional covariance matrices that are developed in a series of work \citep{wibisono2018convexity-heat, wibisono2018convexity-ou, mikulincer2021brownian, mikulincer2023lipschitz, chewi2022entropic, dai2023lipschitz}.

\begin{proof} [Proof of Proposition \ref{prop:vf-bd}]
\begin{itemize}
\item[(a)] By Jung's theorem \citep[Theorem 2.6]{danzer1963helly}, there exists a closed Euclidean ball with radius less than $D := (1/\sqrt{2} ) \mathrm{diam} (\mathrm{supp}(\nu))$ that contains $\mathrm{supp}(\nu)$ in $\sR^d$.
Then the desired bounds hold due to $0 \rmI_d \preceq \Cov (\mathsf{Y} | \mathsf{X}_t = x) \preceq D^2 \rmI_d$ and Eq. \paref{eq:cov-gene-vf}.

\item[(b)] Let $p_1$ be $\beta$-semi-log-convex for some $\beta >0$ on $\sR^d$. Then for any $t\in [0,1)$, the conditional distribution $p(y | t, x)$ is $\left(\beta + \frac{b_t^2}{a_t^2} \right)$-semi-log-convex because
\begin{equation*}
    -\nabla^2_y \log p(y | t, x)
    = -\nabla^2_y \log p_1(y) - \nabla^2_y \log p(t, x | y)
    \preceq \left( \beta + \frac{b_t^2}{a_t^2} \right) \rmI_d.
\end{equation*}
By the Cram{\'e}r-Rao inequality \paref{eq:cri-cov}, we obtain
\begin{equation*}
    \Cov (\mathsf{Y} | \mathsf{X}_t = x) \succeq \left(\beta + \frac{b_t^2}{a_t^2} \right)^{-1} \rmI_d.
\end{equation*}
Therefore, by Eq. \paref{eq:cov-gene-vf}, we obtain
\begin{equation*}
    \nabla_x v(t, x) \succeq \left\{ \left( \frac{\dot{b}_t}{b_t} - \frac{\dot{a}_t}{a_t} \right) \frac{b_t^2}{\beta a_t^2 + b_t^2} + \frac{\dot{a}_t} {a_t} \right\} \rmI_d,
\end{equation*}
which implies
\begin{equation*}
    \nabla_x v(t, x) \succeq \frac{\beta a_t \dot{a}_t + b_t \dot{b}_t} {\beta a_t^2 + b_t^2} \rmI_d.
\end{equation*}
In addition, the bound above can be verified at time $t = 1$ by the definition \paref{eq:vf-time1}.

\item[(c)] Let $p_1$ be $\kappa$-semi-log-concave for some $\kappa \in \sR$. Then for any $t\in [0,1)$, the conditional distribution $p(y | t, x)$ is $\left(\kappa + \frac{b_t^2}{a_t^2} \right)$-semi-log-concave because
\begin{equation*}
    -\nabla^2_y \log p(y | t, x)
    = -\nabla^2_y \log p_1(y) - \nabla^2_y \log p(t, x | y)
    \succeq \left( \kappa + \frac{b_t^2}{a_t^2} \right) \rmI_d.
\end{equation*}
When $t \in \left\{ t : \kappa + \frac{b_t^2}{a_t^2} > 0, t \in (0, 1) \right\}$, by the Brascamp-Lieb inequality \paref{eq:bli-cov}, we obtain
\begin{equation*}
    \Cov (\mathsf{Y} | \mathsf{X}_t = x) \preceq \left(\kappa + \frac{b_t^2}{a_t^2} \right)^{-1} \rmI_d.
\end{equation*}
Therefore, by Eq. \paref{eq:cov-gene-vf}, we obtain
\begin{equation*}
    \nabla_x v(t, x) \preceq
    \left\{ \left( \frac{\dot{b}_t}{b_t} - \frac{\dot{a}_t}{a_t} \right) \frac{b_t^2}{\kappa a_t^2 + b_t^2} + \frac{\dot{a}_t} {a_t} \right\} \rmI_d,
\end{equation*}
which implies
\begin{equation*}
    \nabla_x v(t, x) \preceq \frac{\kappa a_t \dot{a}_t + b_t \dot{b}_t} {\kappa a_t^2 + b_t^2} \rmI_d.
\end{equation*}
Moreover, the bound above can be verified at time $t = 1$ by the definition \paref{eq:vf-time1}.

\item[(d)] Notice that
\begin{align*}
    p(y | t, x)
    &= \frac{p(t, x | y)}{p_t (x)} \frac{\diff (\gamma_{d, \sigma^2} * \rho)}{\diff y} \\
    &= A_{x,t} \int_{\sR^d} \varphi_{z, \sigma^2}(y) \varphi_{\tfrac{x}{b_t}, \tfrac{a_t^2}{b_t^2}} (y) \rho(\diff z),
\end{align*}
where the prefactor $A_{x,t}$ only depends on $x$ and $t$.
Then it follows that
\begin{equation*}
    p(y | t, x) = \int_{\sR^d} \varphi_{ \frac{a_t^2 z +\sigma^2 b_t x}{a_t^2 + \sigma^2 b_t^2}, \frac{\sigma^2 a_t^2}{a_t^2 + \sigma^2 b_t^2} }(y) \tilde \rho(\diff z)
\end{equation*}
where $\tilde \rho$ is a probability measure on $\sR^d$ whose density function is a multiple of $\rho$ by a positive function. It also indicates that $\tilde \rho$ is supported on the same Euclidean ball as $\rho$. To further illustrate $p(y | t, x)$, let $\mathsf{Q} \sim \tilde{\rho}$ and $\mathsf{Z} \sim \gamma_d$ be independent. Then it holds that
\begin{equation*}
    \frac{a_t^2}{a_t^2 + \sigma^2 b_t^2} \mathsf{Q} + \sqrt{\frac{\sigma^2 a_t^2}{a_t^2 + \sigma^2 b_t^2}} \mathsf{Z} + \frac{\sigma^2 b_t}{a_t^2 + \sigma^2 b_t^2} x \sim p(y | t, x).
\end{equation*}
Thus, it holds that
\begin{align*}
    \Cov (\mathsf{Y} | \mathsf{X}_t = x)
    &= \left( \frac{a_t^2}{a_t^2 + \sigma^2 b_t^2} \right)^2 \Cov(\mathsf{Q}) + \frac{\sigma^2 a_t^2}{a_t^2 + \sigma^2 b_t^2} \rmI_d \\
    & \preceq \left\{ \left( \frac{a_t^2}{a_t^2 + \sigma^2 b_t^2} \right)^2 R^2 + \frac{\sigma^2 a_t^2}{a_t^2 + \sigma^2 b_t^2} \right\} \rmI_d.
\end{align*}
By Eq. \paref{eq:cov-gene-vf}, it holds that
\begin{equation*}
    \nabla_x v(t, x) \preceq
    \frac{b_t^2}{a_t^2} \left( \frac{\dot{b}_t}{b_t} - \frac{\dot{a}_t}{a_t} \right) \left( \left( \frac{a_t^2}{a_t^2 + \sigma^2 b_t^2} \right)^2 R^2 + \frac{\sigma^2 a_t^2}{a_t^2 + \sigma^2 b_t^2} \right) \rmI_d + \frac{\dot{a}_t} {a_t} \rmI_d,
\end{equation*}
which implies
\begin{equation*}
    \nabla_x v(t, x) \preceq
    \left\{ \frac{a_t b_t (a_t \dot{b}_t - \dot{a}_t b_t)}{(a_t^2 + \sigma^2 b_t^2)^2} R^2 + \frac{\dot{a}_t a_t + \sigma^2 \dot{b}_t b_t}{a_t^2 + \sigma^2 b_t^2} \right\} \rmI_d.
\end{equation*}
Analogously, due to $\Cov(\mathsf{Q}) \succeq 0 \rmI_d$, a lower bound would be yielded as follows
\begin{equation*}
    \nabla_x v(t, x) \succeq
    \frac{\dot{a}_t a_t + \sigma^2 \dot{b}_t b_t}{a_t^2 + \sigma^2 b_t^2} \rmI_d.
\end{equation*}
Then the results follow by combining the upper and lower bounds.

\item[(e)] The result follows from Corollary \ref{cor:bd-vf-log-Lip}.
\end{itemize}
We complete the proof.
\end{proof}

\begin{proof} [Proof of Corollary \ref{cor:pos-kappa-bd}]
    Let us consider that $\kappa > 0$ which is divided into two cases where $\kappa D^2 \ge 1$ and $\kappa D^2 < 1$.
    On one hand, suppose that the first case $\kappa D^2 \ge 1$ holds.
    By Proposition \ref{prop:vf-bd}, 
    the $\kappa$-based upper bound is tighter, that is,
    \begin{align*}
        \lambda_{\max} (\nabla_x v(t, x)) \le \theta_t := \frac{\kappa a_t \dot{a}_t + b_t \dot{b}_t} {\kappa a_t^2 + b_t^2}.
    \end{align*}
    On the other hand, suppose that the second case $\kappa D^2 < 1$ holds.
    Let $t_1$ be defined in Eq. \paref{eq:critical-cond}.
    Again, by Proposition \ref{prop:vf-bd},
    the $D^2$-based upper bound is tighter over $[0, t_1)$ and the $\kappa$-based upper bound is tighter over $[t_1, 1]$, which is denoted by
    \begin{align*}
        \lambda_{\max} (\nabla_x v(t, x)) \le \theta_t :=
        \begin{cases}
            \frac{b_t^2}{a_t^2} \left( \frac{\dot{b}_t}{b_t} - \frac{\dot{a}_t}{a_t} \right) D^2 + \frac{\dot{a}_t} {a_t}, \ &t \in [0, t_1), \\
            \frac{\kappa a_t \dot{a}_t + b_t \dot{b}_t} {\kappa a_t^2 + b_t^2},  \ &t \in [t_1, 1].
        \end{cases}
    \end{align*}
    This completes the proof.
\end{proof}

\begin{proof} [Proof of Corollary \ref{cor:neg-kappa-bd}]
    Let $\kappa < 0, D < \infty$ such that $\kappa D^2 < 1$ is fulfilled.
    Then an argument similar to the proof of Corollary \ref{cor:pos-kappa-bd} yields the desired bounds.
\end{proof}

\begin{proof} [Proof of Corollary \ref{cor:mog-bd}]
    The result follows from Proposition \ref{prop:vf-bd}-(d).
\end{proof}

\begin{proof} [Proof of Corollary \ref{cor:log-lip}]
    The $L$-based upper and lower bounds in Proposition \ref{prop:vf-bd}-(e) would blow up at time $t = 1$ because the term $(\log (\sqrt{a_t^2 + b_t^2}/b_t))^{-\frac12}$ in $B_t$ goes to $\infty$ as $t \to 1$.
    To ensure the spatial derivative of the velocity field $v(t, x)$ is upper bounded at time $t = 1$, we additionally require the target measure is $\kappa$-semi-log-concave with $\kappa \le 0$.
    Hence, a $\kappa$-based upper bound is available for $t \in (t_0, 1]$ as shown in Proposition \ref{prop:vf-bd}-(c).
    Next, these two upper bounds are combined by choosing any $t_2 \in (t_0, 1)$ first.
    Then we exploit the $L$-based bound over $[0, t_2)$ and $\kappa$-based bound over $[t_0, 1]$.
    This completes the proof.
\end{proof}

\section{Proofs of well-posedness and Lipschitz flow maps}

The proofs of main results in Section \ref{sec:well-posed} are offered in the following.
Before proceeding, let us introduce some definitions and notations about function spaces that are collected in \cite[Chapter 5]{evans2010partial}.
Let $L_{\mathrm{loc}}^1(\sR^d; \sR^{\ell}) := \{ \textrm{locally integrable function } u:\sR^d \to \sR^{\ell} \}$.
For integers $k \ge 0$ and $1 \le p \le \infty$, we define the Sobolev space $W^{k, p}(\sR^d):= \{ u \in L^1_{\mathrm{loc}} (\sR^d) \vert D^{\alpha}u \textrm{ exists and } D^{\alpha}u \in L^p(\sR^d) \textrm{ for } \vert \alpha \vert \le k \}$, where $D^{\alpha}u$ is the weak derivative of $u$.
Then the local Sobolev space $W_{\mathrm{loc}}^{k, p}(\sR^d)$ is defined as the function space such that for any $u \in W_{\mathrm{loc}}^{k, p}(\sR^d)$ and any compact set $\Omega \subset \sR^d$, $u \in W^{k, p}(\Omega)$.
As a result, we denote the vector-valued local Sobolev space by $W_{\mathrm{loc}}^{k, p}(\sR^d; \sR^{\ell})$.
Provided that $v(t, x): [0, 1] \times \sR^d \to \sR^d$, we use $v \in L^1([0, 1]; W^{1, \infty}_{\mathrm{loc}}(\sR^d; \sR^d))$ to indicate that $v$ has a finite $L^1$ norm over $(t, x) \in [0, 1] \times \sR^d$ and $v(t, \cdot) \in W^{1, \infty}_{\mathrm{loc}}(\sR^d; \sR^d)$ for any $t \in [0, 1]$.
Similarly, we say $v \in L^1([0, 1]; L^{\infty}(\sR^d; \sR^d))$ when $v$ has a finite $L^1$ norm over $(t, x) \in [0, 1] \times \sR^d$ and $v(t, \cdot) \in L^{\infty}(\sR^d; \sR^d)$ for every $t \in [0, 1]$. We will use the definitions and notations in the following proof.

\begin{proof} [Proof of Theorem \ref{thm:well-posed}]
Under Assumptions 1 and 2, we claim that the velocity field $v(t, x)$ satisfies
\begin{align*}
    v \in L^1([0, 1]; W^{1, \infty}_{\mathrm{loc}}(\sR^d; \sR^d)), \quad
    \frac{\Vert v \Vert_2}{1 + \Vert x \Vert_2} \in L^1([0, 1]; L^{\infty}(\sR^d; \sR^d)).
\end{align*}
where the first condition indicates the velocity field $v$ is locally bounded and locally Lipschitz continuous in $x$, and the second condition is a growth condition on $v$.
According to the Cauchy-Lipschitz theorem \cite[Remark 2.4]{ambrosio2014continuity}, we have the representation formulae for solutions of the continuity equation.
As a result, there exists a flow $(X_t)_{t \in [0, 1]}$ uniquely solves the IVP \paref{eq:ode-ivp}.
Furthermore, the marginal flow of $(X_t)_{t \in [0, 1]}$ satisfies the continuity equation \paref{eq:cont-eq} in the weak sense.
Then it remains to show the velocity field $v$ is locally bounded and locally Lipschitz continuous in $x$, and satisfies the growth condition.
By the lower and upper bounds given in Proposition \ref{prop:vf-bd}, we know that $v$ is globally Lipschitz continuous in $x$ under Assumptions 1 and 2. Indeed, the global Lipschitz continuity leads to local boundedness and linear growth properties by simple arguments.
More concretely, for any $t\in (0,1)$, it holds that
\begin{align*}
    v(t, 0)
    &= \left( \dot{b}_t - \frac{\dot{a}_t}{a_t} b_t \right) \E[\mathsf{X}_1 | \mathsf{X}_t = 0]
     = \left( \dot{b}_t - \frac{\dot{a}_t}{a_t} b_t \right) \int_{\mathbb R^d} y p(y | t, 0) \diff y \\
    & \lesssim \left( \dot{b}_t - \frac{\dot{a}_t}{a_t} b_t \right) \int_{\sR^d} y p_1(y) a_t^{-d} \exp \left(- \frac{b_t^2 \Vert y \Vert^2_2}{2 a_t^2} \right) \diff y,
\end{align*}
which implies $\Vert v(t, 0) \Vert_2 < \infty$ due to fast growth of the exponential function.
Besides, it holds that $v(0, 0) = (\dot{b}_0 - \tfrac{\dot{a}_0}{a_0}b_0) \E[\mathsf{X}_1 | \mathsf{X}_0 = x] < \infty, v(1, 0) = \dot{a}_1 a_1 s(1, 0) < \infty$.
Then by the boundedness of $\Vert v(t, 0) \Vert_2$ and the global Lipschitz continuity in $x$ over $t \in [0, 1]$, we bound $v(t, x)$ as follows
\begin{align*}
    \Vert v(t, x) \Vert_2
    & \le \Vert v(t, 0) \Vert_2 + \Vert v(t, x) - v(t, 0) \Vert_2 \\
    & \le \Vert v(t, 0) \Vert_2 + \left\{ \sup_{(t, y) \in [0,1] \times \sR^d} \Vert \nabla_y v(t, y) \Vert_{2,2} \right\} \Vert x \Vert_2 \\
    & \lesssim \max\{ \Vert x \Vert_2, 1 \}.
\end{align*}
Hence, the local boundedness and linear growth properties of $v$ are proved.
This completes the proof.
\end{proof}

\begin{proof} [Proof of Theorem \ref{thm:well-posed-bounded}]
The proof is similar to that of Theorem \ref{thm:well-posed}.
\end{proof}

\begin{proof} [Proof of Corollary \ref{cor:time-reve-flow}]
    A well-posed ODE flow has the time-reversal symmetry \citep{lamb1998time}. By Theorem \ref{thm:well-posed}, the desired results are proved.
\end{proof}

\begin{proof} [Proof of Corollary \ref{cor:time-reve-flow-bounded}]
    The proof is similar to that of Corollary \ref{cor:time-reve-flow}.
\end{proof}

\begin{proof} [Proof of Proposition \ref{prop:lip-map}]
    Combining Proposition \ref{prop:vf-bd}-(b), (c), and Lemma \ref{lm:flow-map-Lip-bd}, we complete the proof.
\end{proof}

\begin{proof} [Proof of Proposition \ref{prop:lip-map-mog}]
    Combining Proposition \ref{prop:vf-bd}-(d) and Lemma \ref{lm:flow-map-Lip-bd}, we complete the proof.
\end{proof}

\begin{proof} [Proof of Corollary \ref{cor:ae}]
    By Theorem \ref{thm:well-posed} and Corollary \ref{cor:time-reve-flow}, it holds that
    \begin{align*}
        X_1 \circ X_1^* = X_1 \circ X_1^{-1} = \rmI_d.
    \end{align*}
    This completes the proof.
\end{proof}

\begin{proof} [Proof of Corollary \ref{cor:cycle}]
    By Theorem \ref{thm:well-posed} and Corollary \ref{cor:time-reve-flow}, it holds that
    \begin{align*}
        X_{1, 1} \circ X_{2, 1}^* \circ X_{2, 1} \circ X_{1, 1}^* = X_{1, 1} \circ X_{2, 1}^{-1} \circ X_{2, 1} \circ X_{1, 1}^{-1} = \rmI_d.
    \end{align*}
    This completes the proof.
\end{proof}

\begin{proof} [Proof of Corollary \ref{cor:prepare-bd}]
    Let Assumptions \ref{assump:well-defined} and \ref{assump:geom-prop} hold.
    According to Propositions \ref{prop:lip-map} and \ref{prop:lip-map-mog}, $\Vert \nabla_x X_1(x) \Vert_{2,2}$ is uniformly bounded for Case (i)-(iii) in Assumption \ref{assump:geom-prop}. For Case (iv), the boundedness of $\Vert \nabla_x X_1(x) \Vert_{2,2}$ holds by combining Corollary \ref{cor:log-lip} and Lemma \ref{lm:flow-map-Lip-bd}.
    Using Proposition \ref{prop:vf-bd}, we know that $\Vert \nabla_x v(t, x) \Vert_{2,2}$ is uniformly bounded.
\end{proof}

\begin{proof} [Proof of Proposition \ref{prop:func-ineq}]
    The proof idea is similar to those of \cite[Proposition 1]{ball2003entropy} and \cite[Proposition 18]{cattiaux2014semi}.
    Let $f: \Omega \to \sR$ be of class $C^1$ and $\mathsf{X}_t \sim p_t$.
    First, we consider the case of log-Sobolev inequalities.
    Using that $\mathsf{Z} \sim \gamma_d$ and $\mathsf{X}_1 \sim \nu$ both satisfy the log-Sobolev inequalities in Definition \ref{def:lsi}, we have
    \begin{align*}
        &~~~~~ \E [(f^2 \log f^2) (\mathsf{X}_t)]
        = \E [(f^2 \log f^2) (a_t \mathsf{Z} + b_t \mathsf{X}_1)] \\
        &\le \int \left( \int f^2(a_t z + b_t x) \diff \gamma_d(z) \right) \log \left( \int f^2(a_t z + b_t x) \diff \gamma_d(z) \right) \diff \nu(x) \\
        &~~~~ + \int \left( 2 C_{\mathrm{LS}}(\gamma_d) \int a_t^2 (\Vert \nabla f \Vert_2^2) (a_t z + b_t x) \diff \gamma_d(z) \right) \diff \nu(x) \\
        & \le \left( \int \int f^2(a_t z + b_t x) \diff \gamma_d(z) \diff \nu(x) \right) \log \left( \int \int f^2(a_t z + b_t x) \diff \gamma_d(z) \diff \nu(x) \right) \\
        &~~~~ + 2 C_{\mathrm{LS}}(\nu) \int \Big\Vert \nabla_x \Big(\int f^2(a_t z + b_t x) \diff \gamma_d(z) \Big)^{\frac12} \Big\Vert_2^2 \diff \nu(x) \\
        &~~~~ + 2 a_t^2 C_{\mathrm{LS}}(\gamma_d) \int \int (\Vert \nabla f \Vert_2^2) (a_t z + b_t x) \diff \gamma_d(z) \diff \nu(x) \\
        & \le ~ \E [f^2 (\mathsf{X}_t)] \log \left( \E [f^2 (\mathsf{X}_t)] \right)
        + 2 a_t^2 C_{\mathrm{LS}}(\gamma_d) \E[\Vert \nabla f(\mathsf{X}_t) \Vert_2^2] \\
        &~~~~ + 2 C_{\mathrm{LS}}(\nu) \int \Big\Vert \nabla_x \Big(\int f^2(a_t z + b_t x) \diff \gamma_d(z) \Big)^{\frac12} \Big\Vert_2^2 \diff \nu(x).
    \end{align*}
    By Jensen's inequality and the Cauchy-Schwartz inequality, it holds that
    \begin{align*}
        &~~~~ \int \Big\Vert \nabla_x \Big(\int f^2(a_t z + b_t x) \diff \gamma_d(z) \Big)^{\frac12} \Big\Vert_2^2 \diff \nu(x) \\
        & \le b_t^2 \frac{\int \left( \int (\Vert f \nabla f \Vert_2) (a_t z + b_t x) \diff \gamma_d(z) \right)^2 \diff \nu(x)}
          {\int \int f^2(a_t z + b_t x) \diff \gamma_d(z) \diff \nu(x)} \\
        & \le b_t^2 \int \int (\Vert \nabla f \Vert_2^2) (a_t z + b_t x) \diff \gamma_d(z) \diff \nu(x) \\
        & \le b_t^2 \E[\Vert \nabla f(\mathsf{X}_t) \Vert_2^2].
    \end{align*}
    Hence, combining the equations above and the fact that $C_{\mathrm{LS}}(\gamma_d) \le 1$ \citep{gross1975logarithmic}, it implies that
    \begin{align*}
        \E [(f^2 \log f^2) (\mathsf{X}_t)] - \E [f^2 (\mathsf{X}_t)] \log \left( \E [f^2 (\mathsf{X}_t)] \right)
        \le 2 \left[ a_t^2 + b_t^2 C_{\mathrm{LS}}(\nu) \right] \E[\Vert \nabla f(\mathsf{X}_t) \Vert_2^2],
    \end{align*}
    that is, $C_{\mathrm{LS}}(p_t) \le a_t^2 + b_t^2 C_{\mathrm{LS}}(\nu)$.

    Next, we tackle the case of Poincar{\'e} inequalities by similar calculations.
    Using that $\mathsf{Z} \sim \gamma_d$ and $\mathsf{X}_1 \sim \nu$ both satisfy the Poincar{\'e} inequalities in Definition \ref{def:pi}, we have
    \begin{align*}
        &~~~~~ \E [f^2(\mathsf{X}_t)]
        = \E [f^2 (a_t \mathsf{Z} + b_t \mathsf{X}_1)] \\
        &\le \int \left( \int f(a_t z + b_t x) \diff \gamma_d(z) \right)^2 \diff \nu(x) \\
        &~~~~ + \int \left( C_{\mathrm{P}}(\gamma_d) \int a_t^2 (\Vert \nabla f \Vert_2^2) (a_t z + b_t x) \diff \gamma_d(z) \right) \diff \nu(x) \\
        & \le \left( \int \int f(a_t z + b_t x) \diff \gamma_d(z) \diff \nu(x) \right)^2 \\
        &~~~~ + C_{\mathrm{P}}(\nu) \int \Big\Vert \nabla_x \Big(\int f(a_t z + b_t x) \diff \gamma_d(z) \Big) \Big\Vert_2^2 \diff \nu(x) \\
        &~~~~ + a_t^2 C_{\mathrm{P}}(\gamma_d) \int \int (\Vert \nabla f \Vert_2^2) (a_t z + b_t x) \diff \gamma_d(z) \diff \nu(x) \\
        & \le ~ \left( \E [f (\mathsf{X}_t)] \right)^2
        + \left[ a_t^2 C_{\mathrm{P}}(\gamma_d) + b_t^2 C_{\mathrm{P}}(\nu) \right] \E[\Vert \nabla f(\mathsf{X}_t) \Vert_2^2].
    \end{align*}
    Combining the expression above and $C_{\mathrm{P}}(\gamma_d) \le 1$, it implies that
    \begin{align*}
        \E [f^2(\mathsf{X}_t)] - \left( \E [f (\mathsf{X}_t)] \right)^2
        \le \left[ a_t^2 + b_t^2 C_{\mathrm{P}}(\nu) \right] \E[\Vert \nabla f(\mathsf{X}_t) \Vert_2^2],
    \end{align*}
    that is, $C_{\mathrm{P}}(p_t) \le a_t^2 + b_t^2 C_{\mathrm{P}}(\nu)$.
    This completes the proof.
\end{proof}

\section{Proofs of the stability results}

We provide the proofs of the stability results in Section \ref{sec:app}.

\begin{proof} [Proof of Proposition \ref{prop:stab-iv}]
Let $x_0 = a_0 z + b_0 x_1$ and suppose $X_0(x_0) \sim \mu, X_0(a_0 z) \sim \gamma_{d, a_0^2}$.
According to Corollary \ref{cor:prepare-bd}, the Lipschitz property of $X_1(x)$ implies that $\Vert X_1(x_0) - X_1(a_0 z) \Vert \le C_1 \Vert x_0 - a_0 z \Vert$.
We consider an integral defined by
\begin{equation*}
    I_t := \int \Vert x_0 - a_0 z \Vert^2 \diff \pi_t (X_t(x_0), X_t(a_0 z)),
\end{equation*}
where $\pi_t$ is a coupling made of the joint distribution of $(X_t(x_0), X_t(a_0 z))$.
In particular, the initial value $I_0$ is computed by
\begin{align*}
    I_0 = \int \Vert x_0 - a_0 z \Vert^2 p_0(x_0) \varphi(z) \diff x_0 \diff z
        = \int \Vert b_0 x_1 \Vert^2 p_1(x_1) \diff x_1
        = b_0^2 \E_{\nu} [\Vert \mathsf{X_1} \Vert^2].
\end{align*}
Since $(X_t)_{t \in [0, 1]}$ is well-posed with $X_0(x_0) \sim \mu$ or $X_0(a_0 z) \sim \gamma_{d, a_0^2}$, according to Corollary \ref{cor:change-var}, the coupling $\pi_t$ satisfies the following differential equation
\begin{equation} \label{eq:coupling-grad}
    \partial_t \log \pi_t (X_t(x_0), X_t(a_0 z)) = - \Tr ((\nabla_x v) (t, X_t(x_0))) - \Tr ((\nabla_x v) (t, X_t(a_0 z))).
\end{equation}
Taking the derivative of $I_t$ and using Eq. \paref{eq:coupling-grad}, it implies that
\begin{align*}
    \frac{\diff I_t}{\diff t} \le
    2 \left( \sup_{(s, x) \in [0, 1] \times \sR^d} \Vert \Tr (\nabla_x v(s, x)) \Vert \right) I_t.
\end{align*}
Thanks to $\Vert \Tr (\nabla_x v(s, x)) \Vert \le d \Vert \nabla_x v(s, x) \Vert_{2,2}$, it follows that
\begin{align*}
    \frac{\diff I_t}{\diff t} \le 2 C_2 d I_t, \quad I_0 = b_0^2 \E_{\nu} [\Vert \mathsf{X_1} \Vert^2].
\end{align*}
By Gr{\"o}nwall's inequality, it holds that
$I_t \le b_0^2 \E_{\nu} [\Vert \mathsf{X_1} \Vert^2] \exp (2 C_2 d t)$.
Therefore, we obtain the following $W_2$ bound
\begin{align*}
    W_2({X_1}_{\#} \gamma_{d, a_0^2}, \nu)
    = W_2({X_1}_{\#} \gamma_{d, a_0^2}, {X_1}_{\#} \mu)
    \le C_1 \sqrt{I_1} ~ \le C_1 b_0 \sqrt{\E_{\nu} [\Vert \mathsf{X_1} \Vert^2]} \exp (C_2 d),
\end{align*}
which completes the proof.
\end{proof}

\begin{proof} [Proof of Proposition \ref{prop:stab-vf}]
\begin{itemize}
\item[(i)] On the one hand, by Corollary \ref{cor:prepare-bd}, $v(t, x)$ is Lipschitz continuous in $x$ uniformly over $(t, x) \in [0, 1] \times \sR^d$ with Lipschitz constant $C_2$. By the variational equation \paref{eq:vari-eq-x} and Lemma \ref{lm:flow-map-Lip-bd}, it follows that
\begin{equation*}
    \Vert \nabla_x X_{s,t}(x) \Vert_{2,2}^2 \le \exp \left( 2 \int_s^t \theta_u \diff u \right).
\end{equation*}
Due to the equality \paref{eq:perb-vf-x}, we deduce that
\begin{align*}
        & \Vert X_1(x_0) - Y_1(x_0) \Vert^2 \\
    \le & \left( \int_0^1 \Vert (\nabla_x X_{s,1}) (Y_s(x_0)) \Vert_{2,2} \Vert v(s, Y_s(x_0)) - \tilde{v}(s, Y_s(x_0)) \Vert \diff s \right)^2 \\
    \le & \left( \int_0^1 \Vert (\nabla_x X_{s,1}) (Y_s(x_0)) \Vert_{2,2}^2 \diff s \right)
    \left( \int_0^1 \Vert v(s, Y_s(x_0)) - \tilde{v}(s, Y_s(x_0)) \Vert^2 \diff s \right) \\
    \le & \int_0^1 \exp \left( 2 \int_s^1 \theta_u \diff u \right) \diff s \int_0^1 \Vert v(s, Y_s(x_0)) - \tilde{v}(s, Y_s(x_0)) \Vert^2 \diff s.
\end{align*}
Take expectation and it follows that
\begin{align*}
    W_2^2({Y_1}_{\#} \mu, \nu)
    &\le \E_{x_0 \sim \mu} \left[ \Vert Y_1(x_0) - X_1(x_0) \Vert^2 \right] \\
    &\le \int_0^1 \exp \left( 2 \int_s^1 \theta_u \diff u \right) \diff s \int_0^1 \int_{\sR^d} \Vert v(t, x) - \tilde{v}(t, x) \Vert^2 \tilde{q}_t(x) \diff x \diff t \\
    &\le \varepsilon \int_0^1 \exp \left( 2 \int_s^1 \theta_u \diff u \right) \diff s
\end{align*}
where $\tilde{q}_t$ denotes the density function of ${Y_t}_{\#} \mu$, and we use the assumption that
\begin{equation*}
    \int_0^1 \int_{\sR^d} \Vert v(t, x) - \tilde{v}(t, x) \Vert^2 \tilde{q}_t(x) \diff x \diff t \le \varepsilon
\end{equation*}
in the last inequality.

\item[(ii)] On the other hand, suppose that $\tilde{v}(t, x)$ is Lipschitz continuous in $x$ uniformly over $(t, x) \in [0, 1] \times \sR^d$ with Lipschitz constant $C_3$.
Applying Gr{\"o}nwall's inequality to the variational equation \paref{eq:vari-eq-y}, it follows that
\begin{equation*}
    \Vert \nabla_x Y_{s,t}(x) \Vert_{2,2}^2 \le \exp (2 C_3 (t-s)).
\end{equation*}
By the equality \paref{eq:perb-vf-y}, it holds that
\begin{align*}
        & \Vert Y_1(x_0) - X_1(x_0) \Vert^2 \\
    \le & \left( \int_0^1 \Vert (\nabla_x Y_{s,1}) (X_s(x_0)) \Vert_{2,2} \Vert v(s, X_s(x_0)) - \tilde{v}(s, X_s(x_0)) \Vert \diff s \right)^2 \\
    \le & \left( \int_0^1 \Vert (\nabla_x Y_{s,1}) (X_s(x_0)) \Vert_{2,2}^2 \diff s \right)
    \left( \int_0^1 \Vert v(s, X_s(x_0)) - \tilde{v}(s, X_s(x_0)) \Vert^2 \diff s \right) \\
    \le & \frac{\exp(2 C_3) - 1}{2 C_3} \int_0^1 \Vert v(s, X_s(x_0)) - \tilde{v}(s, X_s(x_0)) \Vert^2 \diff s.
\end{align*}
Taking expectations, it further yields that
\begin{align*}
    W_2^2({Y_1}_{\#} \mu, \nu)
    &\le \E_{x_0 \sim \mu} \left[ \Vert Y_1(x_0) - X_1(x_0) \Vert^2 \right] \\
    &\le \frac{\exp(2 C_3) - 1}{2 C_3} \int_0^1 \int_{\sR^d} \Vert v(t, x) - \tilde{v}(t, x) \Vert^2 p_t(x) \diff x \diff t
\end{align*}
where $X_t(x_0) \sim p_t$.
\end{itemize}
\end{proof}

\section{Time derivative of the velocity field}
In this appendix, we are interested in representing the time derivative of the velocity field via moments of $\mathsf{Y} | \mathsf{X}_t = x$. The result is efficacious for controlling the time derivative with moment estimates, though the computation is somehow tedious.
\begin{proposition}
The time derivative of the velocity field $v(t, x)$ has an expression with moments of $\mathsf{X}_1 | \mathsf{X}_t$ for any $t \in (0, 1)$  as follows
\begin{align*}
    \partial_t v(t, x)
    &= \left( \frac{\ddot{a}_t}{a_t} - \frac{\dot{a}_t^2}{a_t^2} \right) x
       + \left( a_t^2 \frac{\ddot{b}_t}{b_t} - \dot{a}_t a_t \frac{\dot{b}_t}{b_t} - \ddot{a}_t a_t + \dot{a}^2_t \right) \frac{b_t}{a_t^2} M_1 \\
    & \quad + \frac{b_t^2}{a_t^2} \left( \frac{\dot{b}_t}{b_t} - \frac{\dot{a}_t}{a_t} \right) \left( \frac{\dot{b}_t}{b_t} - 2 \frac{\dot{a}_t}{a_t} \right) M^c_2 x
      - \frac{b_t^3}{a_t^2} \left( \frac{\dot{b}_t}{b_t} - \frac{\dot{a}_t}{a_t} \right)^2
      \left( M_3 - M_2 M_1 \right),
\end{align*}
where $M_1 := \E [\mathsf{X}_1 | \mathsf{X}_t = x],  M_2 := \E [\mathsf{X}_1^{\top} \mathsf{X}_1 | \mathsf{X}_t = x], M^c_2 := \Cov(\mathsf{X}_1 | \mathsf{X}_t = x), M_3 := \E [\mathsf{X}_1 \mathsf{X}_1^{\top} \mathsf{X}_1 | \mathsf{X}_t = x]$.
\end{proposition}

\begin{proof}
By direct differentiation, it implies that
\begin{align*}
    \partial_t v(t, x)
    &= \partial_t \left( \frac{\dot{b}_t}{b_t} \right) x + \partial_t \left( \frac{\dot{b}_t}{b_t} a_t^2 - \dot{a}_t a_t \right) s(t, x) + \left( \frac{\dot{b}_t}{b_t} a_t^2 - \dot{a}_t a_t \right) \partial_t s(t, x) \\
    &= \frac{\ddot{b}_tb_t - \dot{b}_t^2}{b_t^2} x + \left( \frac{\ddot{b}_tb_t - \dot{b}_t^2}{b_t^2} a_t^2 + \frac{\dot{b}_t}{b_t} 2 \dot{a}_t a_t - \ddot{a}_t a_t - \dot{a}_t^2 \right) s(t, x) + \left( \frac{\dot{b}_t}{b_t} a_t^2 - \dot{a}_t a_t \right) \partial_t s(t, x).
\end{align*}
We first focus on $\partial_t s(t, x)$. Since $p_t$ satisfies the continuity equation \paref{eq:cont-eq}, it holds that
\begin{align*}
    \partial_t s(t, x)
    &= \nabla_x (\partial_t \log p_t(x)) \\
    &= - \nabla_x \left( \frac{\nabla_x \cdot (p_t(x) v(t, x))}{p_t(x)} \right) \\
    &= - \nabla_x \left( \frac{(\nabla_x p_t(x))^{\top} v(t, x) + p_t(x) (\nabla_x \cdot v(t, x))}{p_t(x)} \right) \\
    &= - \nabla_x \left( s(t, x)^{\top} v(t, x) + \nabla_x \cdot v(t, x) \right) \\
    &= - \left( (\nabla_x s(t, x))^{\top} v(t, x) + (\nabla_x v(t, x))^{\top} s(t, x) + \nabla_x (\nabla_x \cdot v(t, x)) \right) \\
    &= - \left( \nabla_x s(t, x) v(t, x) + \nabla_x v(t, x) s(t, x) + \nabla_x \Tr( \nabla_x v(t, x) ) \right).
\end{align*}
By direct computation, it holds that
\begin{align*}
    & \quad \nabla_x s(t, x) v(t, x) + \nabla_x v(t, x) s(t, x) \\
    &= \nabla_x s(t, x) \left( \frac{\dot{b}_t}{b_t} x + \left( \frac{\dot{b}_t}{b_t} a_t^2 - \dot{a}_t a_t \right) s(t, x) \right)
       + \nabla_x \left( \frac{\dot{b}_t}{b_t} x + \left( \frac{\dot{b}_t}{b_t} a_t^2 - \dot{a}_t a_t \right) s(t, x) \right) s(t, x) \\
    &= \frac{\dot{b}_t}{b_t} \nabla_x s(t, x) x + \left( \frac{\dot{b}_t}{b_t} a_t^2 - \dot{a}_t a_t \right) \nabla_x s(t, x) s(t, x)
       + \frac{\dot{b}_t}{b_t} s(t, x) + \left( \frac{\dot{b}_t}{b_t} a_t^2 - \dot{a}_t a_t \right) \nabla_x s(t, x) s(t, x) \\
    &= \frac{\dot{b}_t}{b_t} s(t, x) + \frac{\dot{b}_t}{b_t} \nabla_x s(t, x) x
       + 2 \left( \frac{\dot{b}_t}{b_t} a_t^2 - \dot{a}_t a_t \right) \nabla_x s(t, x) s(t, x).
\end{align*}
Then we focus on the trace term
\begin{align*}
    & \quad \nabla_x \Tr(\nabla_x v(t, x)) \\
    &= \nabla_x \Tr \left( \left( \frac{\dot{b}_t}{b_t} - \frac{\dot{a}_t}{a_t} \right) \frac{b_t^2}{a_t^2} \Cov (\mathsf{Y} | \mathsf{X}_t = x) + \frac{\dot{a}_t} {a_t} \rmI_d \right) \\
    &= \left( \frac{\dot{b}_t}{b_t} - \frac{\dot{a}_t}{a_t} \right) \frac{b_t^2}{a_t^2} \nabla_x \Tr( \Cov (\mathsf{Y} | \mathsf{X}_t = x) ) \\
    &= \left( \frac{\dot{b}_t}{b_t} - \frac{\dot{a}_t}{a_t} \right) \frac{b_t^2}{a_t^2} \nabla_x \left( \int \Vert y \Vert^2 p(y | t, x) \diff y - \left\Vert \int y p(y | t, x) \diff y \right\Vert^2 \right) \\
    &= \left( \frac{\dot{b}_t}{b_t} - \frac{\dot{a}_t}{a_t} \right) \frac{b_t^2}{a_t^2} \left( \int \Vert y \Vert^2 \nabla_x p(y | t, x) \diff y - 2 \left( \int \nabla_x p(y | t, x) \otimes y \diff y \right) \left( \int y p(y | t, x) \diff y \right) \right),
\end{align*}
where we notice that
\begin{align*}
    \nabla_x p(y | t, x)
    &= \nabla_x \left( \frac{p(t, x | y) p_1(y)}{p_t(x)} \right) \\
    &= \frac{\nabla_x p(t, x | y) p_1(y)} {p_t(x)} - \frac{p(t, x | y) p_1(y)}{p_t(x)} s(t, x) \\
    &= p(y | t, x) \left( \frac{b_t y - x}{a_t^2} - s(t, x) \right).
\end{align*}
For ease of presentation, we introduce the following notations to denote several moments of $\mathsf{Y} | \mathsf{X}_t = x$
\begin{align*}
    M_1 &:= \E [\mathsf{Y} | \mathsf{X}_t = x],       && M_2 := \E [\mathsf{Y}^{\top} \mathsf{Y} | \mathsf{X}_t = x], \\
    M^c_2 &:= \Cov(\mathsf{Y} | \mathsf{X}_t = x),    && M_3 := \E [\mathsf{Y} \mathsf{Y}^{\top} \mathsf{Y} | \mathsf{X}_t = x].
\end{align*}
By Tweedie's formula in Lemma \ref{lm:tw-formula}, it yields $s(t, x) = \frac{b_t}{a_t^2} M_1 - \frac{1}{a_t^2} x$.
By this expression of $s(t, x)$, it yields
\begin{align*}
    & \quad \nabla_x s(t, x) v(t, x) + \nabla_x v(t, x) s(t, x) \\
    &= \frac{\dot{b}_t}{b_t} s(t, x) + \frac{\dot{b}_t}{b_t} \nabla_x s(t, x) x
       + 2 \left( \frac{\dot{b}_t}{b_t} a_t^2 - \dot{a}_t a_t \right) \nabla_x s(t, x) s(t, x) \\
    &= \frac{\dot{b}_t}{b_t} \left( \frac{b_t}{a_t^2} M_1 - \frac{1}{a_t^2} x \right)
       + \frac{\dot{b}_t}{b_t} \left( \frac{b_t^2}{a_t^4} M^c_2 - \frac{1}{a_t^2} \rmI_d \right) x \\
    & \quad + 2 \left( \frac{\dot{b}_t}{b_t} a_t^2 - \dot{a}_t a_t \right) \left( \frac{b_t^2}{a_t^4} M^c_2 - \frac{1}{a_t^2} \rmI_d \right) \left( \frac{b_t}{a_t^2} M_1 - \frac{1}{a_t^2} x \right) \\
    &= -2 \frac{\dot{a}_t}{a_t^3} x + \frac{b_t}{a_t^2} \left( 2 \frac{\dot{a}_t}{a_t} - \frac{\dot{b}_t}{b_t} \right) M_1
       + \frac{b_t^2}{a_t^4} \left( 2 \frac{\dot{a}_t}{a_t} - \frac{\dot{b}_t}{b_t} \right) M^c_2 x
       + 2 \frac{b_t^3}{a_t^4} \left( \frac{\dot{b}_t}{b_t} - \frac{\dot{a}_t}{a_t} \right) M^c_2 M_1
\end{align*}
and $\nabla_x p(y | t, x) = \frac{b_t}{a_t^2} \left(y - M_1 \right) p(y | t, x)$.
Therefore, we obtain
\begin{align*}
     & \quad \int \Vert y \Vert^2 \nabla_x p(y | t, x) \diff y - 2 \left( \int \nabla_x p(y | t, x) \otimes y \diff y \right) \left( \int y p(y | t, x) \diff y \right) \\
     &= \int \Vert y \Vert^2 \frac{b_t}{a_t^2} \left(y - M_1 \right) p(y | t, x) \diff y
        - 2 \left( \int \frac{b_t}{a_t^2} \left(y - M_1 \right) \otimes y p(y | t, x) \diff y \right) \left( \int y p(y | t, x) \diff y \right) \\
     &= \frac{b_t}{a_t^2} \left[ \int \Vert y \Vert^2 y p(y | t, x) \diff y
     - \left( \int \Vert y \Vert^2 p(y | t, x) \diff y \right) M_1 \right. \\
     & \left.~~~~~~~~ -2 \left( \int y \otimes y p(y | t, x) \diff y
     - M_1 \otimes \int y p(y | t, x) \diff y \right) \left( \int y p(y | t, x) \diff y \right) \right] \\
     &= \frac{b_t}{a_t^2} \left( M_3
     - M_2 M_1
     - 2 M^c_2 M_1 \right).
\end{align*}
Combining the equations above, we obtain
\begin{align*}
    \partial_t v(t, x)
    &= \frac{\ddot{b}_tb_t - \dot{b}_t^2}{b_t^2} x
       + \left( \frac{\ddot{b}_tb_t - \dot{b}_t^2}{b_t^2} a_t^2 + \frac{\dot{b}_t}{b_t} 2 \dot{a}_t a_t - \ddot{a}_t a_t - \dot{a}^2_t \right) \left( \frac{b_t}{a_t^2} M_1 - \frac{1}{a_t^2} x \right) \\
    & ~~~~~ - \left( \frac{\dot{b}_t}{b_t} a_t^2 - \dot{a}_t a_t \right)
      \left[ -2 \frac{\dot{a}_t}{a_t^3} x + \frac{b_t}{a_t^2} \left( 2 \frac{\dot{a}_t}{a_t} - \frac{\dot{b}_t}{b_t} \right) M_1 \right. \\
    &~~~~~~~~~~~~~~~~~~~~~~~~~~~~~~ \left. + \frac{b_t^2}{a_t^4} \left( 2 \frac{\dot{a}_t}{a_t} - \frac{\dot{b}_t}{b_t} \right) M^c_2 x
        + 2 \frac{b_t^3}{a_t^4} \left( \frac{\dot{b}_t}{b_t} - \frac{\dot{a}_t}{a_t} \right) M^c_2 M_1 \right] \\
    & ~~~~~ - \left( \frac{\dot{b}_t}{b_t} a_t^2 - \dot{a}_t a_t \right) \left( \frac{\dot{b}_t}{b_t} - \frac{\dot{a}_t}{a_t} \right) \frac{b_t^3}{a_t^4} (M_3 - M_2 M_1 - 2 M^c_2 M_1) \\
    &= \left( \frac{\ddot{a}_t}{a_t} - \frac{\dot{a}_t^2}{a_t^2} \right) x
       + \left( a_t^2 \frac{\ddot{b}_t}{b_t} - \dot{a}_t a_t \frac{\dot{b}_t}{b_t} - \ddot{a}_t a_t + \dot{a}^2_t \right) \frac{b_t}{a_t^2} M_1 \\
    & \quad + \frac{b_t^2}{a_t^2} \left( \frac{\dot{b}_t}{b_t} - \frac{\dot{a}_t}{a_t} \right) \left( \frac{\dot{b}_t}{b_t} - 2 \frac{\dot{a}_t}{a_t} \right) M^c_2 x
      - \frac{b_t^3}{a_t^2} \left( \frac{\dot{b}_t}{b_t} - \frac{\dot{a}_t}{a_t} \right)^2
      \left( M_3 - M_2 M_1 \right).
\end{align*}
Then we complete the proof.
\end{proof}

\section{Functional inequalities and Tweedie’s formula}
This appendix is devoted to an exposition of functional inequalities and Tweedie’s formula that would assist in our proof.

For a probability measure $\mu$ on a compact set $\Omega \subset \sR^d$, we define the variance of a function $f \in L^2(\Omega, \mu)$ as
\begin{align*}
    \mathrm{Var}_{\mu}(f) := \int_{\Omega} f^2 \diff \mu - \left( \int_{\Omega} f \diff \mu \right)^2.
\end{align*}
Moreover, for a probability measure $\mu$ on a compact set $\Omega \subset \sR^d$ and any positive integrable function $f: \Omega \to \sR$ such that $\int_{\Omega} f \Vert \log f \Vert \diff \nu < \infty$, we define the entropy of $f$ as
\begin{align*}
    \mathrm{Ent}_{\mu}(f) := \int_{\Omega} f \log f \diff \mu - \int_{\Omega} f \diff \mu \log \left( \int_{\Omega} f \diff \mu \right).
\end{align*}

\begin{definition} [Log-Sobolev inequality] \label{def:lsi}
    A probability measure $\mu \in \mathcal{P}(\Omega)$ is said to satisfy a log-Sobolev inequality with constant $C > 0$, if for all functions $f: \Omega \to \sR$, it holds that
    \begin{align*}
        \mathrm{Ent}_{\mu}(f^2) \le 2 C \int_{\Omega} \Vert \nabla f \Vert_2^2 \diff \mu.
    \end{align*}
    The best constant $C > 0$ for which such an inequality holds is referred to as the log-Sobolev constant $C_{\mathrm{LS}}(\mu)$.
\end{definition}

\begin{definition} [Poincar{\'e} inequality] \label{def:pi}
    A probability measure $\mu \in \mathcal{P}(\Omega)$ is said to satisfy a Poincar{\'e} inequality with constant $C > 0$, if for all functions $f: \Omega \to \sR$, it holds that
    \begin{align*}
        \mathrm{Var}_{\mu}(f) \le C \int_{\Omega} \Vert \nabla f \Vert_2^2 \diff \mu.
    \end{align*}
    The best constant $C > 0$ for which such an inequality holds is referred to as the Poincar{\'e} constant $C_{\mathrm{P}}(\mu)$.
\end{definition}

Finally, for ease of reference, we present Tweedie's formula that was first reported in \citet{robbins1956empirical}, and then was used as a simple empirical Bayes approach for correcting selection bias \citep{efron2011tweedie}. Here, we use Tweedie's formula to link the score function with the expectation conditioned on an observation with Gaussian noise.

\begin{lemma} [Tweedie's formula]
\label{lm:tw-formula}
Suppose that $\mathsf{X} \sim \mu$ and $\epsilon \sim \gamma_{d, \sigma^2}$.
Let $\mathsf{Y} = \mathsf{X} + \epsilon$ and $p(y)$ be the marginal density of $\mathsf{Y}$.
Then $\E[\mathsf{X} | \mathsf{Y} = y] = y + \sigma^2 \nabla_y \log p(y)$.
\end{lemma}

\vskip 0.2in
\bibliography{GIF_arXiv_new1.bib}

\end{document}